\documentclass{article}

\PassOptionsToPackage{numbers, compress}{natbib}
\usepackage[preprint]{neurips_2022}

\usepackage[utf8]{inputenc} % allow utf-8 input
\usepackage[T1]{fontenc}    % use 8-bit T1 fonts
\usepackage{hyperref}       % hyperlinks
\usepackage{url}            % simple URL typesetting
\usepackage{booktabs}       % professional-quality tables
\usepackage{amsfonts}       % blackboard math symbols
\usepackage{nicefrac}       % compact symbols for 1/2, etc.
\usepackage{microtype}      % microtypography
\usepackage{xcolor}         % colors
\usepackage{graphicx}

\usepackage{amsmath,amssymb,amsfonts}
\usepackage{algorithmic}
\usepackage{algorithm}

\usepackage{subfigure}

\usepackage{multirow}

\usepackage{amsthm}
\newtheorem{theorem}{Theorem}

\usepackage{wrapfig}

\usepackage{color}
\usepackage{colortbl}
\definecolor{graywhite}{RGB}{235,235,235}

\title{FL-Tuning: Layer Tuning for Feed-Forward Network in Transformer}

\author{
\footnotesize
Jingping Liu$^{1}$,
Yuqiu Song$^{1}$,
Kui Xue$^{2}$,
Hongli Sun$^{1}$,
Chao Wang$^{3}$,
Lihan Chen$^{3}$,\\
\footnotesize
\textbf{Haiyun Jiang}$^{4}$\textbf{,}
\textbf{Jiaqing Liang}$^{3}$\textbf{,}
\textbf{Tong Ruan}$^{1}$\thanks{Corresponding author}
\\
\footnotesize
$^{1}$School of Information Science and Engineering, East China University of Science and Technology  \\
\footnotesize
$^{2}$Shanghai Artificial Intelligence Laboratory, Shanghai, China\\
\footnotesize
$^{3}$Shanghai Key Laboratory of Data Science, School of Computer Science, Fudan University\\
\footnotesize
$^{4}$Tencent AI Lab, Shenzhen, China, \\ 
\footnotesize
\{jingpingliu, ruantong\}@ecust.edu.cn
}
  
\begin{document}

\maketitle

\begin{abstract}

Prompt tuning is an emerging way of adapting pre-trained language models to downstream tasks. However, the existing studies are mainly to add prompts to the input sequence. This way would not work as expected due to the intermediate multi-head self-attention and feed-forward network computation, making model optimization not very smooth. Hence, we propose a novel tuning way called layer tuning, aiming to add learnable parameters in Transformer layers. Specifically, we focus on layer tuning for feed-forward network in the Transformer, namely FL-tuning. It introduces additional units into the hidden layer of each feed-forward network. We conduct extensive experiments on the public CLUE benchmark. The results show that: 1) Our FL-tuning outperforms prompt tuning methods under both full-data and few-shot settings in almost all cases. In particular, it improves accuracy by 17.93\% (full-data setting) on WSC 1.0 and F1 by 16.142\% (few-shot setting) on CLUENER over P-tuning v2. 2) Our FL-tuning is more stable and converges about 1.17 times faster than P-tuning v2. 3) With only about 3\% of Transformer's parameters to be trained, FL-tuning is comparable with fine-tuning on most datasets, and significantly outperforms fine-tuning (e.g., accuracy improved by 12.9\% on WSC 1.1) on several datasets. The source codes are available at \url{https://github.com/genggui001/FL-Tuning}.
\end{abstract}

\section{Introduction}
% \begin{wrapfigure}{r}{0.5\textwidth}
% \includegraphics[width=0.5\textwidth]{./Figure1.png}\
% \caption{The main improvements made by P-tuning, P-tuning v2 and FL-tuning in the Transformer's encoder respectively.}
% \label{fig:compare}
% \end{wrapfigure}
% Pre-trained language models (PLMs), such as ELMo~\cite{DBLP:journals/corr/abs-1802-05365} and BERT~\cite{DBLP:journals/corr/abs-1810-04805}, have been a successful approach for enhancing the performance of universal language processing (NLP) tasks~\cite{liu2021gpt}. \textbf{Fine-tuning} is usually used to adapt these general-purpose models to various downstream tasks, which updates all pre-trained parameters for specific downstream tasks. Although fine-tuning has good performance, it needs to update and store all the parameters during training, resulting in a larger memory requirement~\cite{DBLP:journals/corr/abs-2004-12651,DBLP:journals/corr/abs-1905-05583}.

Pre-trained language models (PLMs), such as ELMo~\cite{DBLP:conf/naacl/PetersNIGCLZ18} and BERT~\cite{DBLP:conf/naacl/DevlinCLT19}, have become increasingly important in natural language processing. The popular way to accommodate general-purpose PLMs to specific downstream tasks is to \textbf{FINE TUNE} them by updating all the parameters. As a result, it is necessary to store a modified copy of full-size model parameters for each task~\cite{DBLP:conf/acl/LiL20}. However, this would be prohibitively expensive when applying the model to a large number of tasks
% However, this needs to store a separate copy of the fine-tuned model parameters for each task, which is prohibitively expensive
~\cite{DBLP:conf/emnlp/ChenHCCLY20,DBLP:conf/cncl/SunQXH19,DBLP:conf/emnlp/LesterAC21}.

\textbf{PROMPT TUNING} is an emerging way to adapt PLMs to downstream tasks, which \textbf{\textit{adds prompts to the input sequence}} and feeds the new input to PLMs in the pre-training task. In this way, all PLM parameters are frozen and only the task-specific prompt is updated.
% we freeze PLMs' pre-trained parameters and only store the small task-specific prompt. 
\textbf{Discrete prompt tuning} is the first method that leverages text tokens as the prompt to the original input~\cite{DBLP:conf/emnlp/ShinRLWS20}. For example, in text classification, 
% 一种常见的prompt tuning方法是将A和B拼接起来
a common prompt tuning method is to concatenate an input (e.g., ``I haven't published a paper.'') with the prompt ``I felt [MASK]'' and ask PLMs to fill the masked token with ``happy'' or ``sad''. Since the PLMs are continuous from an optimization point of view, it is difficult to achieve the optimum with discrete prompts~\cite{liu2021p}. \textbf{Continuous prompt tuning} is thus proposed to replace text tokens with trainable embeddings, which outperforms discrete prompt tuning on many tasks~\cite{DBLP:conf/emnlp/LesterAC21,liu2021p,liu2021gpt}. However, the impact of input prompts before the first transformer layer will gradually weaken due to the multiple intermediate layers' computation~\cite{liu2021p}. To address this problem, \textbf{deep prompt tuning}~\cite{liu2021p} is proposed, which takes continuous embeddings as the prompt to the input sequence of each layer in PLMs. 
Although this method has shown better performance on many tasks, its essence is still only adding a few parameters as prompts to the input. 
These prompts would not work as expected due to the intermediate multi-head self-attention and feed-forward network computation, making model optimization not very smooth. This situation leads to limited performance gain, unstable training, and slow convergence.

To address the above problem, we propose a novel tuning way, namely \textbf{LAYER TUNING}, for adapting PLMs to downstream tasks. Different from prompt tuning, layer tuning is to \textbf{\textit{add learnable parameters to the Transformer layers}} with original PLM parameters frozen. Transformer~\cite{DBLP:conf/nips/VaswaniSPUJGKP17} mainly contains two sub-layers: multi-head self-attention and feed-forward network (FFN). In this paper, 
we mainly focus on \textbf{L}ayer \textbf{Tuning} for \textbf{F}eed-forward network in the Transformer, namely \textbf{FL-tuning}. As shown in Figure \ref{fig:compare}, it aims at introducing additional hidden units into each FFN layer. In other words, FL-tuning expands the dimensions of the weight matrices ($W_1$ and $W_2$ in Eq. (\ref{equation:Att})) and bias ($b_1$) of the linear layers in FFN.
% layer tuning有两种加法：一种是attention
% % 在FFN上面tuning的原因是它占据了模型2/3的参数，对它进行微调可能会有更好的结果。
The reason for tuning on FFN is that it accounts for about $2/3$ of the number of parameters in the Transformer encoder~\cite{DBLP:conf/emnlp/GevaSBL21}, and it is expected to obtain better results.
% 在A和B上分别进行tuning的对比将在Section5.3中阐述。
% 此外，我们将该思想应用在A上，并通过实验验证了在FFN上进行X的优越性
% Feed Forward layers account for about $2/3$ of the number of parameters in Transformer's encoder~\cite{geva2020transformer}
% \textcolor{red}{The reason for tuning on FFN is that it is the closest to the output layer, which can minimize the instability of model training caused by the calculation of the intermediate layers.} 
Due to the inconvenience of directly operating on the expanded weight matrices and biases, we split them into fixed and learnable parts. Hence, we can perform independent operations on the two parts to simplify the implementation process. The feasibility of this ``split'' idea is rigorously proved theoretically in Section \ref{section:Implementation}. In addition, we also prove that the influence of the trainable weight matrices and bias on model performance is independent of their positions in the expanded ones. That is, the learnable part can be placed at any position in the expanded weight matrices and bias.
% demonstrate that the influence of the learnable matrix/bias on the model is independent of its position in the expanded matrix/bias. That is, the former matrix/bias can be placed at any position in the latter.

\textbf{Contributions.} The contributions in this paper are summarized as follows:
\begin{itemize}
    \item To the best of our knowledge, we are the first to propose layer tuning for adapting PLMs to downstream task. The most distinguish characteristic is that it adds learnable parameters in Transformer layers.
    
    \item We propose a method of layer tuning for feed-forward network in the Transformer, namely FL-tuning. Our tuning method is more stable and converges faster than P-tuning v2.
    
    % \textcolor{red}{It only uses learnable parameters (about 0.3\% of the number of whole model's parameters) to equivalently affect about 68\% parameters. This method converges faster and is more stable than deep prompt tuning.}
    
    \item We conduct extensive experiments on 7 downstream tasks and 11 NLU datasets. The results show that our FL-tuning outperforms prompt tuning in almost all cases. In addition, with only about 3\% of Transformer's parameters to be trained, it is comparable with fine-tuning on most datasets, and significantly outperforms fine-tuning on several datasets.
    % In particularly, our FL-tuning brings nearly 18\% improvement in accuracy over P-tuning v2 on WSC 1.0 dataset of the pronoun disambiguation task.
\end{itemize}

\begin{figure}
\begin{center}
\includegraphics[width=0.95\textwidth]{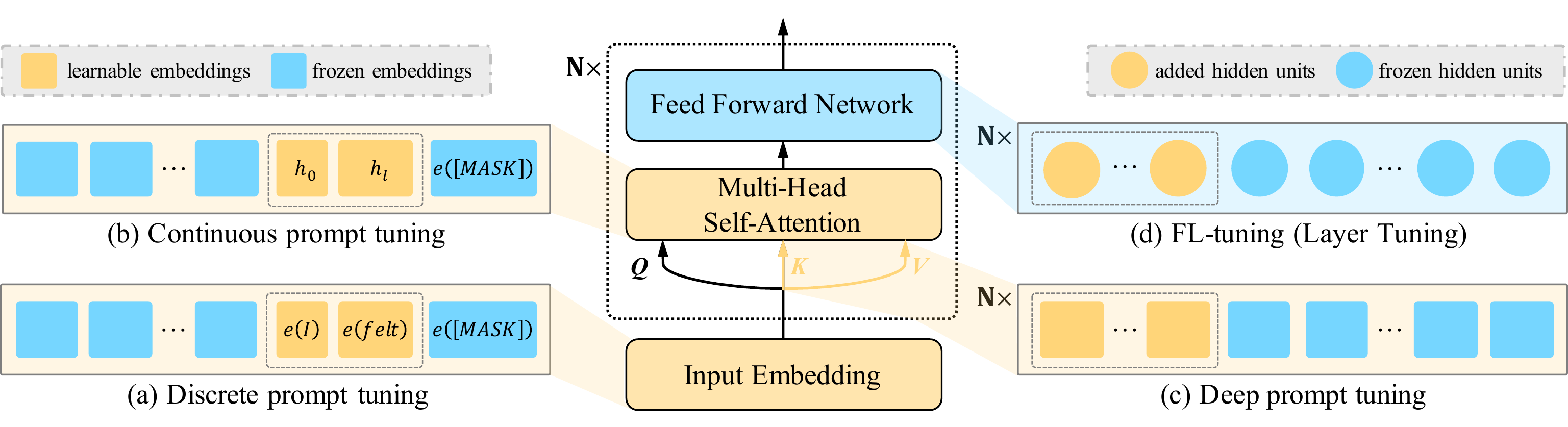}
\caption{The difference between FL-tuning and prompt tuning. The former introduces learnable parameters to each FFN in the Transformer, while the latter adds tokens/embeddings to the input.}
\label{fig:compare}
\end{center}
\end{figure}

\section{Related Work}
\label{section:related-work}

% \subsection{Prompt tuning}

% Prompt tuning在层前的输入序列加入prompt，并且已经被证明在许多NLP应用中有效
% Prompt tuning主要分三类
% First, Discrete prompt tuning。例如
% Second, Continuous prompt tuning。例如P-tuning
% Third, Deep Prompt tuning。例如Prefix-tuning、P-tuning v2
% In this paper, 我们借鉴了deep prompt tuning 的思想，但与上述的方法不同的是，我们专注于在层内进行改进

\textbf{Prompt tuning} adapts PLMs to downstream tasks by adding prompts to the input sequence and has been verified effective in many NLP applications~\cite{DBLP:conf/nips/BrownMRSKDNSSAA20,jiang2020can,han2021ptr,DBLP:conf/naacl/QinE21,DBLP:conf/acl/GaoFC20}. These methods can be divided into three categories. First, discrete prompt tuning directly generates the results without changing the pre-trained model parameters based only on the text tokens. LM-BFF~\cite{DBLP:conf/acl/GaoFC20} introduces T5~\cite{DBLP:journals/jmlr/RaffelSRLNMZLL20} into the template search process to generate tokens. Typical 
studies also include GPT-3~\cite{DBLP:conf/nips/BrownMRSKDNSSAA20} and LAMA~\cite{DBLP:conf/emnlp/PetroniRRLBWM19}.
% However, discrete prompt tuning leads to suboptimal performance due to the continuous nature of neural networks~\cite{liu2021p,zhang2021differentiable}. 
Second, the appearance of P-tuning~\cite{liu2021gpt} turns on continuous prompt tuning work, which replaces text tokens with trainable embeddings to overcome the limitation of sub-optimal performance in the previous methods. P-tuning~\cite{liu2021gpt} introduces continuous embeddings learned by LSTM to the original sequence of input word embeddings. 
% \textcolor{red}{The impact of the continuous prompt can be propagated upward to all transformer activation layers and rightward to subsequent tokens.}
Prefix-tuning~\cite{DBLP:conf/acl/LiL20} designs task-specific trainable prefixes for natural language generation tasks. 
% This method avoids long-range dependencies and has more expressive power than P-tuning's embedding-only.
% , it will gradually weaken due to the intermediate computation. Moreover, the amount of parameters to tune is limited as prompts are only added to the input embedding sequence.
Third, deep prompt tuning is proposed to solve the challenges of continuous prompt tuning, which are the lack of universality across scales and tasks. 
% which are limited amount of parameters to tune and limited stability. 
P-tuning v2~\cite{liu2021p} applies continuous embeddings as the prompts to the input sequence of each layer in PLMs. 
% for each layer in PLMs and the prefix trainable embeddings of each layer are independent. 
% \textcolor{red}{It demonstrates deep prompt tuning can be comparable to Fine-tuning universally across scales and tasks.} 
Unlike the above studies, we aim to propose a novel PLMs' tuning way, namely layer tuning, which adds learnable parameters to the Transformer layers.
\textbf{Transformer} has been proved to be an effective architecture for PLMs~\cite{DBLP:conf/nips/VaswaniSPUJGKP17}. It is an encoder-decoder structure, which is composed of multi-head self-attention and FFN.
Much effort is dedicated to improve the Transformer's capability~\cite{DBLP:conf/naacl/FanGLWWJDZH21} and can be divided into three categories. First, some studies attempt to design a novel self-attention architecture in the Transformer~\cite{katharopoulos2020transformers,vyas2020fast,wang2021predictive,DBLP:conf/acl/XiongZS18,shazeer2020talking}. Linformer~\cite{wang2020linformer} proposes a new self-attention mechanism, which reduces the complexity of self-attention to $\textit{O(n)}$ in both time and space. 
% First, the success of Transformer is largely due to the success of self-attention. There is a lot of research on multi-head self-attention in Transformer~\cite{katharopoulos2020transformers,vyas2020fast,wang2021predictive,zhang2018accelerating,shazeer2020talking}. Linformer propose a new self-attention mechanism, which reduces the complexity of self-attention to $\textit{O(n)}$ in both time and space~\cite{wang2020linformer}.
Second, although FFN is just a multi-layer perceptron, it accounts for about $2/3$ of the number of parameters in the Transformer encoder. Kformer~\cite{DBLP:journals/corr/abs-2201-05742} improves the ability of PLMs by incorporating external knowledge into the FFN in the Transformer. 
% Furthermore, recent researches propose that FN layers in Transformer-based PLMs store factual knowledge~\cite{DBLP:journals/corr/abs-2104-08696} and can be viewed as unnormalized Key-Value Memories~\cite{DBLP:conf/emnlp/GevaSBL21}.
% Recent work proposes that FFN layers in Transformer-based PLMs can bear high knowledge intensity, store factual knowledge~\cite{dai2021knowledge}. Geva et al. argue that FFN can be viewd as unnormalized Key-Value Memories~\cite{DBLP:conf/emnlp/GevaSBL21}. Kformer improve the ability of PLMs by incorporating external knowledge into the FFN in Transformer~\cite{DBLP:journals/corr/abs-2201-05742}. 
Third, in addition to the separate researches on the two sub-layers, there are also many studies to modify the Transformer structure. Fan et al.~\cite{DBLP:conf/naacl/FanGLWWJDZH21} strength the Transformer with a dynamic mask attention network before the multi-head self-attention. 
% \textcolor{red}{Fan et al. add a new layer named dynamic mask attention network with a learnable mask matrix before each multi-head self-attention~\cite{DBLP:conf/naacl/FanGLWWJDZH21}.} 
Press et al.~\cite{DBLP:conf/acl/PressSL20} study the impact of the combined order of the two sub-layers in the Transformer on model performance. Different from the above work on sub-layers or adding new sub-layers to the Transformer, this paper focuses on introducing hidden units to each FFN layer for PLMs' tuning.

\section{Preliminaries}
\label{section:Preliminaries}
In this section, we briefly review the multi-head self-attention and feed-forward network in the Transformer architecture. Besides, we give the definitions of three kinds of prompt tuning, including discrete prompt tuning, continuous prompt tuning, and deep prompt tuning.

% ~\cite{DBLP:conf/emnlp/ShinRLWS20} ~\cite{liu2021gpt} ~\cite{liu2021p}

\subsection{Transformer}
\label{subsection:Transformer}
The Transformer model is based on an encoder-decoder structure. Each encoder (decoder) is composed of a stack of identical blocks, containing multi-head self-attention and feed-forward network. These two sub-layers are formulated as follows: 
\begin{equation}
\label{equation:Att}
    Attention({Q},{K},{V})=Softmax(\frac{QK^T}{\sqrt{d_k}}){V},\ \ \ 
    FFN(X)=ReLU(XW_1+b_1){W_2}+b_2,
\end{equation}
where $X\in \mathbb{R}^{d_u\times d_m}$, $Q=XW^Q$, $K=XW^K$, $V=XW^V$, $W^Q\in \mathbb{R}^{d_{m}\times d_k}$, $W^K\in \mathbb{R}^{d_{m}\times d_k}$, $W^V\in \mathbb{R}^{d_{m}\times d_v}$, $W_1\in \mathbb{R}^{d_{m}\times d_o}$, $W_2\in \mathbb{R}^{d_o\times d_{m}}$ are weight matrices, $b_1\in \mathbb{R}^{1\times d_o}$ and $b_2\in \mathbb{R}^{1\times d_{m}}$ are bias, and $ReLU(x)=max(0,x)$.

% The classical Transformer model consists of an encoder and a decoder, each of which is a stack of identical blocks. The block is mainly composed of Multi-Head Self-Attention and Feedforward Neural Network~\cite{lin2021survey}.The two sub-layers are formulated as follows:
% \begin{small}
% \begin{align}
% \label{equation:Att}
%     {SA}({Q},{K},{V})=Softmax(\frac{QK^T}{\sqrt{d_k}}){V},\ \ \ 
%     FFN(X)=ReLU(XW_1+b_1)W_2+b_2,
% \end{align}
% \end{small}
% where $Q=XW^Q, K=XW^K, V=XW^V$, and $W^Q, W^K, W^V, W_1, W_2$ are the learnable metrics, $ReLU(x)=max(0,x)$.

\subsection{Prompt tuning}
\textbf{\emph{Discrete prompt tuning}} leverages text tokens as prompts to adapt PLMs to downstream tasks. For example, to classify the sentiment of $S$=``I haven't published a paper.'' as \textit{happy} or \textit{sad}, we design a prompt ``I felt [MASK]'' for $S$. The input embedding for PLMs is defined as:
\begin{equation}
    [e(S),e(I),e(felt),e([MASK])].
\end{equation}

\textbf{\emph{Continuous prompt tuning}} adds the trainable embeddings to the input ones. Continuing the previous example, the prompt ``I felt [MASK]'' is replaced with continuous embeddings [$h_0,...,h_l$], where $l$ is the prompt length. The input sequence is thus formulated as:
\begin{equation}
    [e(S), h_0, ..., h_l, e([MASK])].
\end{equation}

\textbf{\emph{Deep prompt tuning}}, also known as P-tuning v2~\cite{liu2021p}, takes the prefix embeddings as the prompt to the input sequence in each Transformer layer. P-tuning v2 essentially concatenates two matrices $E_0$ and $E_1$ to the key matrix $K$ and the value matrix $V$ respectively in each multi-head attention layer, which is defined as:
\begin{equation}
    {Attention_{PV2}}({Q},{K},{V},{E_0},{E_1}) = Softmax(\frac{Q[E_0 \bot K]^T}{\sqrt{d_k}}){[E_1 \bot V]},
\end{equation}
where the operation ``$\bot$'' means the row-wise concatenation, $E_0$ and $E_1$ are trainable and other parameters are frozen.
% w1: w1' ：表示A和B依据行向量进行拼接
% C表示A和B依据【行】向量进行拼接
% （https://numpy.org/doc/stable/reference/generated/numpy.vstack.html这是numpy官网对于拼接函数的解释）

% \textbf{\emph{Prompting}} leverages discrete natural language tokens as prompts to reformulate downstream tasks. For example, to classify the sentiment of $\textbf{X}$ =``I haven't published a paper.'' as happy or sad, a prompt ``I felt [MASK]'' is designed for $\textbf{X}$.  The input embedding sequence is shown as Eq. \ref{equation:Prompting}. However, PLMs are continuous in essence, continuous prompt is a superior choice compared to discrete prompt.
% \begin{small}
% \begin{equation}
% \begin{aligned}
% \label{equation:Prompting}
%     \relax [e(X),e(I),e(felt),e([MASK])]. 
% \end{aligned}
% \end{equation}
% \end{small}

% \textbf{\emph{P-tuning}} adds trainable continuous embeddings to the original sequence of input word embeddings (Cf. Figure \ref{fig:compare}(a)). Continuing the previous example, the prompt ``I felt [MASK]'' is replaced with continuous embedding  [$h_0,...,h_i$] and the input sequence is shown as Eq. \ref{equation:P-tuning}. P-tuning transforms the natural language of tokens into a continuous parameter optimization problem, turning on continuous prompt work. However, as the layer deepens, prompts go through more nonlinear activation functions, resulting in unexpected effects. 
% \begin{small}
% \begin{equation}
% \label{equation:P-tuning}
%     \relax [e(X),h_0,...,h_i,e([MASK])].
% \end{equation}
% \end{small}

\begin{figure}
\begin{center}
\includegraphics[width=0.95\textwidth]{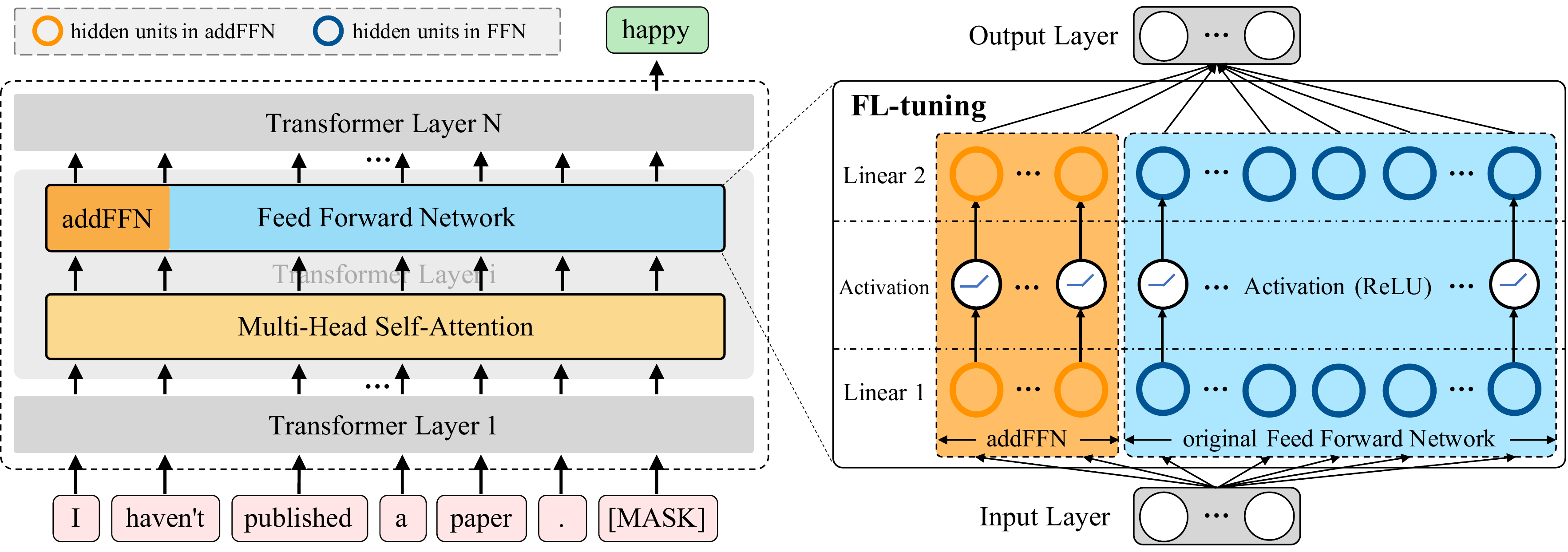}
\caption{The model architecture of FL-tuning. FL-tuning introduces additional hidden units to each FFN layer, and the units of each layer are independent. The orange circles are hidden units in addFFN, the blue circles are hidden units in original FFN.}
\label{fig:model}
\end{center}
\end{figure}

% \textbf{\emph{P-tuning v2}}~\cite{liu2021p} affects the multi-head attention layer of the transformer (Cf. Figure \ref{fig:compare}(b)). It adds continuous prompts as prefix tokens to the input sequence in different multi-head attention layers. P-tuning v2 essentially splices two trainable matrices $E_0$ and $E_1$ on the keys matrix $K$ and the values matrix $V$ in each multi-head attention layer, which is defined as Eq. \ref{equation:pv2 multi-head attention}. PV2 shows improvements for models across scales and hard sequence NLU tasks. However, PV2 converges slowly and is sensitive to hyperparameters, which is described in Section \ref{section:Experiments}.
% \begin{small}
% \begin{equation}
% \label{equation:pv2 multi-head attention}
% \begin{aligned}
%     {SA_{PV2}}({Q},{K},{V},{E_0},{E_1}) = Softmax(\frac{Q[E_0,K]^T}{\sqrt{d_k}}){[E_1,V]},
% \end{aligned}
% \end{equation}
% \end{small}
% where $E_0$ and $E_1$ are trainable and other parameters are frozen.

\section{Methodology}
\label{section:FL-tuning}
% In this paper, we propose a novel method of FL-tuning to apply PLMs to downstream tasks. Next, we will introduce the details of FL-tuning and give the theoretic proofs of FL-tuning and 
% its simplified implementation process.
In this section, we first introduce the details of FL-tuning. Then, we give the theoretical proofs of its equivalent implementation.
% When we adapt PLMs to downstream tasks, fine-tuning strategies are often employed. We propose a novel method, namely FL-tuning, focusing on Layer Tuning for FFN. In this section, we introduce the details of FL-tuning and give the theoretical proof of the equivalent implementation process.

% And theoretical proof of the equivalent implementation process is given.

\subsection{FL-tuning}
\label{section:Architecture}

In the Transformer~\cite{DBLP:conf/nips/VaswaniSPUJGKP17} architecture, FFN is a multi-layer perceptron with one hidden layer, consisting of two linear transformations with an ReLU activation in their middle~\cite{DBLP:conf/nips/VaswaniSPUJGKP17}. In essence, FL-tuning adds a number of hidden units in each FFN layer, as shown in Figure \ref{fig:model}. Note that the added hidden units in different layers are independent to each other. Formally, we change $W_1$ and ${W_2}$ in Eq. (\ref{equation:Att}) to [$W'_1$:$W_1$] and $[W'_2 \bot W_2]$ as follows:
\begin{equation}
\label{FNNFP}
    FFN_{FL}(X)=ReLU(X[W'_1:W_1]+[b'_1:b_1])[W'_2 \bot W_2]+b_2,
\end{equation}
where the operations ``$:$'' and ``$\bot$'' stand for the column- and row-wise concatenation, respectively. Only the parameters of $W'_1\in \mathbb{R}^{d_m\times d_a}$, $b'_1\in \mathbb{R}^{1\times d_a}$, and $W'_2\in \mathbb{R}^{d_a\times d_m}$ can be trained and all parameters of the Transformer-based PLMs are frozen.

\subsection{Implementation}
\label{section:Implementation}
However, it is inconvenient to directly utilize Eq. (\ref{FNNFP}) to implement FL-tuning because the learnable and fixed parameters are mixed in matrices. Therefore, we keep the original FFN hidden layer and add another small FFN hidden layer (referred to as addFFN). 
The equivalence is established by the  following theorem.

\begin{theorem}
\label{theorem: prefix}
In FL-tuning, FFN can be split into:
\begin{equation}
\label{split}
    {FFN_{FL}}(X)={addFFN}(X)+{FFN}(X),
\end{equation}
where ${addFFN}(X)=ReLU(XW'_1+b'_1){W'_2}$.
\end{theorem}

\begin{proof}
Based on the distributive property of matrix multiplication, we make a further derivation of Eq. (\ref{FNNFP}) as below:
\begin{equation}
\label{equation:derivation}
    FFN_{FL}(X)=ReLU([XW'_1+b'_1:XW_1+b_1]){[W'_2 \bot W_2]}+b_2.
\end{equation}
Let $H$ be the output of the hidden layer in FFN and $H'$ be the output of the addFFN's hidden layer, i.e.,
\begin{equation}
    \label{H}
    H=ReLU(XW_1+b_1),\ \ \ H'=ReLU(XW'_1+b'_1).
\end{equation}
% Furthermore, we substitute $H$, $H'$ into the Eq. (\ref{equation:derivation}) and reformulate it into the following equation:
Eq. (\ref{equation:derivation}) can be rewritten as:
\begin{equation}
\label{equation:reformulated}
    FFN_{FL}(X)=[H':H]{[W'_2 \bot W_2]}+b_2
    =H'{W'_2}+H{W_2}+b_2
    =addFFN(X)+FFN(X),
\end{equation}
which proves the theorem.
% From the above derivation, Eq. (\ref{split}) holds.
\end{proof}

In the proof of Theorem \ref{theorem: prefix}, additional hidden units are introduced into FFN as the prefix (at the beginning of the FFN hidden layer). Alternatively, we also have the choices of the infix (in the middle of the FFN hidden layer) and suffix (at the end of the FFN hidden layer) to add hidden units. However, these three cases are equivalent, as stated in the following theorem.
% Next, we analyze the two cases of the infix (in the middle of FFN hidden layer) and suffix (at the end of FFN hidden layer.), and obtain the following conclusion.

\begin{theorem}
\label{theorem: location}
In FL-tuning, Eq. (\ref{split}) is not affected by the position of $addFFN(\cdot)$. That is, Eq. (\ref{split}) holds regardless of whether $addFFN(\cdot)$ is prefix, infix, or suffix.
\end{theorem}

\begin{proof} The proof of this theorem is divided into the following three cases:

\textit{\textbf{Prefix:}} In the proof of Theorem \ref{theorem: prefix}, we have ${FFN_{FL}}(X)={addFFN}(X)+{FFN}(X)$.

\textit{\textbf{Infix:}} In this case, $W_1$, $b_1$, and ${W_2}$ can be split into $[W_{1p}:W_{1s}]$, $[b_{1p}:b_{1s}]$, and $[W_{2p} \bot W_{2s}]$. Then, we formulate FL-tuning in the infix form as follows:
\begin{equation}
\begin{aligned}
    FFN_{FL}(X)=&ReLU(X[W_{1p}:W'_1:W_{1s}]+[b_{1p}:b'_1:b_{1s}]){[W_{2p} \bot W'_2 \bot W_{2s}]}+b_2.  \\
    =&[H_{1p}:H':H_{1s}]{[W_{2p} \bot W'_2 \bot W_{2s}]}+b_2 =H_{1p}{W_{2p}}+H'{W'_2}+H_{1s}{W_{2s}}+b_2    \\
    =&H'{W'_2}+[H_{1p}:H_{1s}]{[W_{2p} \bot W_{2s}]}+b_2 =H'{W'_2}+H{W_2}+b_2    \\
    =&addFFN(X)+FFN(X).
\end{aligned}
\label{equation:Infix}
\end{equation}
\textbf{\emph{Suffix}:} The proof of this case is similar to that of Theorem \ref{theorem: prefix}. That is,
\begin{equation}
\begin{aligned}
    FFN_{FL}(X)=&ReLU(X[W_1:W'_1][b_1:b'_1]){[W_2 \bot W'_2]}+b_2 \\ =&[H:H']{[W_2 \bot W'_2]}+b_2
    =H{W_2}+H'{W'_2}+b_2    \\
    =&FFN(X)+addFFN(X)
\end{aligned}
\label{equation:Suffix}
\end{equation}
In summary, Eq. (\ref{split}) is not influenced by the position of $addFFN(\cdot)$.
\end{proof}

% However, the above $FFN_{LP}$ is difficult to implement. Because it is obviously inconvenient to set a part of the matrix with trainable parameters and another part with frozen parameters. Therefore, we keep the original FFN layer and add another small FFN layer (which can be called addFFN) to serve as LP-tuning's prompts in Figure \ref{fig:demonstrate}. The implementation process can be formally expressed as follows:
% \begin{small}
% \begin{equation}
% \begin{aligned}
%     {FFN_{FL}}(X)={addFFN}(X)+{FFN}(X),
% \end{aligned}
% \end{equation}
% \end{small}
% where ${addFFN}(X)=ReLU(XW'_1+b'_1){W'_2}+b_2$.

% Proof of equivalent implementation. $b_1$, $b'_1$ and $b_2$ participate in the operation under the \emph{broadcasting mechanism}. The calculation contains matrix multiplication operation (not dot product) and matrix addition operation (the corresponding elements do addition). Actually, whether the matrix operation corresponding to addFFN is placed to the left or right of the ``$\textbf{+}$'' does not affect the validity of the equation.

% \begin{figure}
%     \centering
%     \includegraphics[width=0.9\textwidth]{./Prove.pdf}
%     \caption{Implementation of FL-tuning, where $FFN_{FL}(\cdot)=addFFN(\cdot)+FFN(\cdot)$. In the training phase, the parameters of FFN are frozen and only the parameters of addFFN can be tuned.}
%     % . $addFFN(\cdot)+FFN(\cdot)=FFN_{FL}(\cdot)$. The parameters of FFN are frozen and only the parameters of addFFN can be tuned.}
%     \label{fig:demonstrate}
% \end{figure}

\section{Experiments}
\label{section:Experiments}

In this section, we investigate our tuning method over three PLMs and seven NLU tasks. The experimental results and detailed analysis demonstrate the superior performance, stable training, and faster convergence of FL-tuning.

\subsection{Experimental Setup}
\label{subsection:dataset}

\textbf{Backbones.} We perform experiments on three Transformer-based PLMs, including RoBERTa~\cite{DBLP:journals/corr/abs-1907-11692}, NEZHA~\cite{wei2019nezha}, and RoFormer~\cite{su2021roformer}. We carefully tune the hyperparameters in experiments based on RoBERTa and apply the optimal ones to NEZHA and RoFormer. Hyperparameters used in fine-tuning, P-tuning v1~\cite{liu2021gpt}, P-tuning v2~\cite{liu2021p}, and FL-tuning over the three PLMs are shown in Table \ref{tab:hyperparameters}.

\begin{table}[t]
\caption{Hyperparameters (BS: batch size; SL: sequence length; LR: learning rate; Epo: epoch) used in PLMs' tuning methods on CLUE benchmark and the statistics of CLUE.}
\label{tab:hyperparameters}
\centering
\resizebox{\textwidth}{!}
{
    \begin{tabular}{l|cccc|cccc|cccc|cccc|ccc}
    \toprule[2pt]
    \multirow{2}{*}{}  & \multicolumn{4}{c}{\textbf{Fine-tuning}}   & \multicolumn{4}{c}{\textbf{P-tuning v1}}   & \multicolumn{4}{c}{\textbf{P-tuning v2}}    & \multicolumn{4}{c|}{\textbf{FL-tuning}}  &\multirow{2}{*}{Train} &\multirow{2}{*}{Dev}  &\multirow{2}{*}{Test}   \\ 
    & BS    & SL   & LR    & Epo     & BS    & SL   & LR    & Epo     & BS    & SL   & LR    & Epo  & BS    & SL   & LR    & Epo   \\ 
    \midrule[2pt]
    IFLYTEK & 32    & 128   & 1e-5  & 30    & 128   & 128   & 6e-4  & 40       & 32    & 128   & 2e-4  & 50    & 32    & 128   & 5e-5  & 30    & 12.1K   & 2.5K  & 2.6K    \\
    TNEWS   & 32    & 128   & 1e-5  & 40    & 128   & 128   & 4e-4  & 30       & 32    & 128   & 1e-3  & 80    & 32    & 128   & 5e-5  & 30    & 266K    & 57K   & 57K     \\
    WSC & 32    & 128   & 5e-5  & 50    & 128   & 128   & 6e-4  & 20       & 32    & 128   & 3e-2  & 80    & 32    & 128   & 3e-4  & 40    & 532 & 104   & 143     \\
    AFQMC & 32  & 128   & 1e-5  & 15    & 128   & 128   & 6e-4  & 30        & 32    & 128   & 8e-3  & 30    & 32    & 128   & 2e-4  & 10    & 34.3K   & 4.3K  & 3.8K  \\
    CMNLI   & 32    & 128   & 1e-5  & 10    & 128   & 128   & 8e-4  & 40    & 32    & 128   & 1e-2  & 100   & 32    & 128   & 1e-4  & 30    & 391.7K  & 12.4K   & 13.8K \\
    OCNLI   & 32    & 128   & 1e-5  & 30    & 128   & 128   & 4e-4  & 80        & 32    & 128   & 1e-2  & 150   & 32    & 128   & 5e-4  & 30    & 50K & 3K    & 3K    \\
    CSL & 32    & 128   & 1e-5  & 30    & 32    &128    & 128   & 6e-4  & 50        & 128   & 9e-3  & 50    & 32    & 128   & 8e-5  & 100   & 532 & 104   & 143   \\
    CLUENER  & 32   & 256   & 2e-5  & 30    & -  & -  & -  & -      & 32    & 256   & 8e-2  & 50    & 32    & 256   & 1e-4  & 30    & 10.7K   & 1.3K  & 1.3K    \\
    C3   & 4 & 512   & 1e-5  & 50       & 16    & 512   & 1e-3  & 50      & 4 & 512   & 1e-2  & 100   & 4 & 512   & 1e-4  & 50    & 11.8K   & 3.8K   & 3.8K  \\
    CHID    & 16    & 64    & 1e-5  & 30    & 128   & 64   & 4e-4   & 30       & 16    & 64    & 2e-3  & 80    & 16    & 64    & 8e-5  & 30    & 84.7K   & 3.2K  & 3.2K   \\
    CMRC2018    & 16    & 512   & 5e-5  & 10    & -  & -  & -  & -       & 16    & 512   & 3e-2  & 50    & 16    & 512   & 2e-4  & 20  & 12.5K & 1.2K  & 4K   \\
    \bottomrule[2pt]
    \end{tabular}
}
\end{table}

% TC
\begin{table}[htpb]
\caption{Comparison results of different tuning methods on TC task. We highlight the best of FT, PV1, PV2, and FL in dark gray and use light gray for the next best. If FL outperforms PV2, the improvement over PV2 is reflected in red below. The difference between TNEWS 1.0 and TNEWS 1.1 is that their validation sets are different.}
\label{tab:TC}
\centering
\resizebox{\textwidth}{!}
{
    \begin{tabular}{lcccccccccccccc}
    \toprule[2pt]
    \multirow{3}{*}{}  & \multicolumn{14}{c}{\textbf{Text Classification (TC)}}  \\ 
    \cmidrule(r){2-15}
    & \multicolumn{4}{c}{IFLYTEK} &  & \multicolumn{4}{c}{TNEWS 1.0}    &    & \multicolumn{4}{c}{TNEWS 1.1}  \\

    & FT    & PV1    & PV2   & FL    &     & FT   & PV1   & PV2   & FL   &     & FT    & PV1    & PV2   & FL      \\ 
    \midrule[2pt]
    RoBERTa & 61.000    & 54.310  & \cellcolor{graywhite}61.730 & \cellcolor{lightgray}62.150     &          & \cellcolor{graywhite}57.170   & 55.700    & 56.770  & \cellcolor{lightgray}57.510       &              & 56.960    & 55.370   & \cellcolor{graywhite}58.130    & \cellcolor{lightgray}59.750     \\

    &   &   & \cellcolor{graywhite} & \cellcolor{lightgray}\textcolor{red}{(+0.420)}        &        & \cellcolor{graywhite}     &   &   & \cellcolor{lightgray}\textcolor{red}{(+0.740)}       &          &  &    & \cellcolor{graywhite}    & \cellcolor{lightgray}\textcolor{red}{(+1.620)}    \\

    NEZHA  & 59.080     & 55.650   & \cellcolor{lightgray}62.810    & \cellcolor{graywhite}60.770     &      & \cellcolor{lightgray}58.510   & 56.360   & \cellcolor{graywhite}58.190    & 57.940           &           & 58.990      & 58.440      & \cellcolor{lightgray}61.300    & \cellcolor{graywhite}60.590 \\

    &   &   & \cellcolor{lightgray}    & \cellcolor{graywhite}                   &                 & \cellcolor{lightgray}   &   & \cellcolor{graywhite}    &                      &                &     &   & \cellcolor{lightgray}    & \cellcolor{graywhite}   \\

    RoFormer & 61.000   & 54.620     & \cellcolor{graywhite}61.620  & \cellcolor{lightgray}62.000       &      & \cellcolor{lightgray}57.840     & 55.370   & \cellcolor{graywhite}57.380  &57.070          &               & 55.210  & 54.530   & \cellcolor{lightgray}58.100  & \cellcolor{graywhite}57.700  \\

    &   &   & \cellcolor{graywhite}  & \cellcolor{lightgray}\textcolor{red}{(+0.380)}      &         & \cellcolor{lightgray}     &    & \cellcolor{graywhite}  &                          &                 &     &    & \cellcolor{lightgray}  & \cellcolor{graywhite}       \\
    \bottomrule[2pt]
    \end{tabular}
}
\end{table}

\begin{table}[!htpb]
\caption{Comparison results of different tuning methods on PD and NER tasks. FL-tuning achieves the best performance on the two tasks. In particular, it improves accuracy by 12.76\% and 17.93\% on WSC 1.0 over FT and PV2, respectively.}
% Evaluation metrics for PD are Accuracy, F1 for NER.}
\label{tab:PDandNER}
\centering
\resizebox{\textwidth}{!}
{
    \begin{tabular}{lcccccccccccccc}
    \toprule[2pt]
    \multirow{3}{*}{}   & \multicolumn{9}{c}{\textbf{Pronoun Disambiguation (PD)}}  &  & \multicolumn{4}{c}{\textbf{Name Entity Recognition (NER)}} \\ 
    \cline{2-10} \cline{12-15}
    & \multicolumn{4}{c}{WSC 1.0}   &  & \multicolumn{4}{c}{WSC 1.1}     &   & \multicolumn{4}{c}{CLUENER}  \\

    & FT  & PV1  & PV2   & FL            &              & FT  & PV1    & PV2   & FL          &     & FT  & PV1   & PV2   & FL  \\ 
    \midrule[2pt]
    RoBERTa     & \cellcolor{graywhite}80.690   & 78.620     & 79.660 & \cellcolor{lightgray}82.070            &                & 78.240    & \cellcolor{graywhite}79.910     & 78.590 & \cellcolor{lightgray}80.540             &             & \cellcolor{graywhite}79.596    & -     & 79.583 & \cellcolor{lightgray}80.479 \\

    & \cellcolor{graywhite}     &   &  & \cellcolor{lightgray}\textcolor{red}{(+2.410)}         &                &  & \cellcolor{graywhite} &     & \cellcolor{lightgray}\textcolor{red}{(+1.950)}                   &             & \cellcolor{graywhite}   &       &  & \cellcolor{lightgray}\textcolor{red}{(+0.896)} \\

    NEZHA       & \cellcolor{graywhite}68.620   &   61.380   & 63.450   & \cellcolor{lightgray}81.380         &                & 66.080     & 51.630   & \cellcolor{graywhite}69.810    & \cellcolor{lightgray}78.980             &             & \cellcolor{graywhite}80.636   & -   & 80.296    & \cellcolor{lightgray}80.866   \\

    & \cellcolor{graywhite}     &    &    & \cellcolor{lightgray}\textcolor{red}{(+17.930)}     &               &   &      & \cellcolor{graywhite}    & \cellcolor{lightgray}\textcolor{red}{(+9.170)}                   &             & \cellcolor{graywhite}  &     &     & \cellcolor{lightgray}\textcolor{red}{(+0.570)}    \\

    RoFormer    & 71.720    & 62.410  & \cellcolor{graywhite}75.170    & \cellcolor{lightgray}78.970        &              & \cellcolor{graywhite}74.510    & 56.370   & 71.520  & \cellcolor{lightgray}81.240            &             & \cellcolor{graywhite}79.395   & -   & 79.229  & \cellcolor{lightgray}80.839  \\

    &   &    & \cellcolor{graywhite}  & \cellcolor{lightgray}\textcolor{red}{(+3.800)}         &       & \cellcolor{graywhite}  &      &   & \cellcolor{lightgray}\textcolor{red}{(+9.720)}                      &             & \cellcolor{graywhite}  &   &   & \cellcolor{lightgray}\textcolor{red}{(+1.610)}   \\
    \bottomrule[2pt]
    \end{tabular}
}
\end{table}

% SS+NLI+KR
\begin{table}[ht]
\caption{Comparison results of different tuning methods on SS, NLI, and KR tasks. FL-tuning is better than prompt tuning methods. Although FT slightly outperforms FL, the number of trainable parameters in FL is only about 3\% of Transformer's parameters.}
% Comparisons with Fine-tuning, PV1, and PV2 on Sentence Pair Tasks (Semantic Similarity, Natural Language Interfere and Keyword Recognition). All elevation metrics are Accuracy.}
\label{tab:SSandNLIandKR}
\centering
\resizebox{\textwidth}{!}
{
\begin{tabular}{lccccccccccccccccccc}
\toprule[2pt]
\multirow{3}{*}{}  & \multicolumn{4}{c}{\textbf{Semantic Similarity (SS)}}    &  & \multicolumn{9}{c}{\textbf{Natural Language Inference (NLI)}} & & \multicolumn{4}{c}{\textbf{Keyword Recognition (KR)}}        \\ 
\cline{2-5} \cline{7-15}  \cline{17-20}
& \multicolumn{4}{c}{AFQMC} &   & \multicolumn{4}{c}{CMNLI}     &   & \multicolumn{4}{c}{OCNLI}  & & \multicolumn{4}{c}{CSL}   \\

& FT  & PV1  & PV2   & FL      &     & FT   & PV1   & PV2   & FL    &     & FT     &    PV1  & PV2   & FL    &    & FT  & PV1     & PV2   & FL \\ 
\midrule[2pt]
RoBERTa & \cellcolor{graywhite}72.570   & 69.850 & 72.030 & \cellcolor{lightgray}73.370            &             & \cellcolor{lightgray}80.430  & 67.470 & 79.990 & \cellcolor{graywhite}80.370      &           & \cellcolor{graywhite}72.900  & 64.630  & 72.370 & \cellcolor{lightgray}73.530                 &                & \cellcolor{lightgray}85.330     & 81.000   & 84.000    & \cellcolor{graywhite}84.930 \\
& \cellcolor{graywhite} &  &  & \cellcolor{lightgray}\textcolor{red}{(+1.340)}          &                & \cellcolor{lightgray}    &  &  & \cellcolor{graywhite}\textcolor{red}{(+0.380)} &         & \cellcolor{graywhite}    &   &  & \cellcolor{lightgray}\textcolor{red}{(+1.160)}          &                & \cellcolor{lightgray}   &   &     & \cellcolor{graywhite}\textcolor{red}{(+0.930)}\\

NEZHA & 73.970  & 69.770  & \cellcolor{lightgray}74.330    & \cellcolor{graywhite}74.180      &                & \cellcolor{lightgray}81.450    & 71.950  & 80.160    & \cellcolor{graywhite}81.160        &     & \cellcolor{lightgray}75.300  & 67.600   & 74.100    & \cellcolor{graywhite}74.470      &                & 84.500  & 82.130  & \cellcolor{lightgray}84.900    & \cellcolor{graywhite}84.670\\
&  &  & \cellcolor{lightgray}    & \cellcolor{graywhite}           &                & \cellcolor{lightgray} &   &     & \cellcolor{graywhite}\textcolor{red}{(+1.000)}    &            & \cellcolor{lightgray}  &  &     & \cellcolor{graywhite}\textcolor{red}{(+0.370)}            &               &  &   & \cellcolor{lightgray}    & \cellcolor{graywhite} \\

RoFormer & \cellcolor{lightgray}73.820  & 69.960   & 71.920  & \cellcolor{graywhite}72.990           &                  & \cellcolor{lightgray}81.060  & 64.840  & 78.770  & \cellcolor{graywhite}80.920     &      & \cellcolor{lightgray}73.970   & 63.670   & 71.800  & \cellcolor{graywhite}72.100              &                  & \cellcolor{graywhite}84.130    & 82.730   & 83.830  & \cellcolor{lightgray}85.030\\
& \cellcolor{lightgray}  &   &   & \cellcolor{graywhite}\textcolor{red}{(+1.070)}             &                  & \cellcolor{lightgray}  &  &   & \cellcolor{graywhite}\textcolor{red}{(+2.150)}              &             & \cellcolor{lightgray}  &  &   & \cellcolor{graywhite}\textcolor{red}{(+0.300)}                &               & \cellcolor{graywhite}    &    &   & \cellcolor{lightgray}\textcolor{red}{(+1.200)}\\
\bottomrule[2pt]
\end{tabular}
}
\end{table}

% Reading Comprehension
\begin{table}[t]
\caption{Comparison results of different tuning methods on MRC task. FL significantly outperforms PV2 in almost all cases and is comparable with FT.}
% Comparisons with Fine-tuning, PV1, and PV2 on Reading Comprehension (RC). Evaluation metrics for RC are Accuracy.}
\label{tab:RC}
\centering
\resizebox{\textwidth}{!}
{
    \begin{tabular}{lccccccccccccccccccc}
    \toprule[2pt]
    \multirow{3}{*}{}   & \multicolumn{19}{c}{\textbf{Machine Reading Comprehension (MRC)}} \\ 
    \cline{2-20}
    & \multicolumn{4}{c}{C3 1.0}    &  & \multicolumn{4}{c}{C3 1.1} &  & \multicolumn{4}{c}{CHID}   &    & \multicolumn{4}{c}{CMRC2018}  \\

    & FT    & PV1   & PV2   & FL    &    & FT    & PV1   & PV2   & FL   &    & FT     & PV1   & PV2   & FL    &    & FT    & PV1   & PV2   & FL   \\ 
    \midrule[2pt]
    RoBERTa & \cellcolor{lightgray}67.650   & 50.800    & 64.210 & \cellcolor{graywhite}67.060  &   & \cellcolor{graywhite}73.600    & 48.000 & 70.090 & \cellcolor{lightgray}74.340    &    & \cellcolor{lightgray}87.544   & 50.605   & 82.981  & \cellcolor{graywhite}87.246     &    & \cellcolor{graywhite}77.500     & -   & 76.300    & \cellcolor{lightgray}77.900   \\
    & \cellcolor{lightgray} &  &   & \cellcolor{graywhite}\textcolor{red}{(+2.850)}     &     & \cellcolor{graywhite} &  &   & \cellcolor{lightgray}\textcolor{red}{(+4.250)}   &     & \cellcolor{lightgray} &  &   & \cellcolor{graywhite}\textcolor{red}{(+4.265)}   &     & \cellcolor{graywhite}   &  &    & \cellcolor{lightgray}\textcolor{red}{(+1.600)}    \\

    NEZHA & \cellcolor{graywhite}70.990     & 54.550    & 69.810    & \cellcolor{lightgray}71.220   &      & \cellcolor{lightgray}76.120  & 50.220    & 73.290    & \cellcolor{graywhite}73.540     &     & \cellcolor{lightgray}89.250  & 63.911   & 85.557    & \cellcolor{graywhite}88.961   &     & \cellcolor{lightgray}76.350   & -  & 72.650    & \cellcolor{graywhite}75.900   \\
    & \cellcolor{graywhite}     &  &   & \cellcolor{lightgray}\textcolor{red}{(+1.410)}     &     & \cellcolor{lightgray}   &  &   & \cellcolor{graywhite}\textcolor{red}{(+0.250)}     &     & \cellcolor{lightgray}  &  &   & \cellcolor{graywhite}\textcolor{red}{(+3.404)}  &    & \cellcolor{lightgray}    &   &   & \cellcolor{graywhite}\textcolor{red}{(+3.250)}    \\

    RoFormer & \cellcolor{lightgray}66.650  & 53.030    & 60.380  & \cellcolor{graywhite}65.930     &     & \cellcolor{graywhite}72.550  & 51.320  & 67.880     & \cellcolor{lightgray}73.720   &    & \cellcolor{graywhite}86.759     & 54.276   & 81.916  & \cellcolor{lightgray}86.953  &  & \cellcolor{graywhite}73.500     & -   & \cellcolor{lightgray}74.600  & 72.050  \\
    & \cellcolor{lightgray}     &   &   & \cellcolor{graywhite}\textcolor{red}{(+5.550)}   &    & \cellcolor{graywhite}  &   &   & \cellcolor{lightgray}\textcolor{red}{(+5.840)}     &     & \cellcolor{graywhite}     &   &   & \cellcolor{lightgray}\textcolor{red}{(+5.037)}     &     & \cellcolor{graywhite}     &   & \cellcolor{lightgray}  &       \\
    \bottomrule[2pt]
    \end{tabular}
}
\end{table}

% Few-shot
\begin{table}[!ht]
\caption{Comparison results of layer/prompt tuning methods under few-shot setting. Our FL-tuning achieves the best results in almost all cases. In particular, it improves F1 of CLUENER by more than 10\% on average over PV2.}
\label{tab:Few-shot}
\centering
\resizebox{\textwidth}{!}
{
    \begin{tabular}{ccccccccccccccccc}
    \toprule[2pt]
    \multirow{2}{*}{}   & \multicolumn{3}{c}{\textbf{TNEWS 1.0}}    &    &  \multicolumn{3}{c}{\textbf{CMNLI}}    &   & \multicolumn{3}{c}{\textbf{CLUENER}}  &    & \multicolumn{3}{c}{\textbf{CHID}}\\ 
    \cline{2-4} \cline{6-8} \cline{10-12} \cline{14-16} 

    & PV1   & PV2   & FL   &   & PV1   & PV2   & FL   &   & PV1   & PV2   & FL &   & PV1   & PV2   & FL  \\ 
    \midrule[2pt]
    % 10  & 42.770     & \cellcolor{lightgray}47.660     & \cellcolor{graywhite}47.320         &    & 33.030    & \cellcolor{lightgray}33.570    & \cellcolor{graywhite}33.240    &   & - & \cellcolor{graywhite}15.016    & \cellcolor{lightgray}21.648     &       & 50.106    & \cellcolor{graywhite}51.398    & \cellcolor{lightgray}52.570  \\

    20  & 42.990    & \cellcolor{graywhite}48.730      & \cellcolor{lightgray}49.150           &  & 34.060    & \cellcolor{graywhite}44.790      & \cellcolor{lightgray}44.900          &       & -     & \cellcolor{graywhite}23.211  & \cellcolor{lightgray}27.525        &                 & 51.226   & \cellcolor{graywhite}52.131  & \cellcolor{lightgray}52.811       \\
    
    &   &  \cellcolor{graywhite}    & \cellcolor{lightgray}\textcolor{red}{(+0.420)}   & &   & \cellcolor{graywhite}    & \cellcolor{lightgray}\textcolor{red}{(+0.11)} &  &   -   &   \cellcolor{graywhite}    & \cellcolor{lightgray}\textcolor{red}{(+4.314)}   &   &   & \cellcolor{graywhite}    & \cellcolor{lightgray}\textcolor{red}{(+0.680)}   \\
        %  &   &   & 

    % 30  & \cellcolor{graywhite}44.010    & \cellcolor{lightgray}49.440      & \cellcolor{lightgray}49.440           &  & 35.220    & \cellcolor{graywhite}47.280      & \cellcolor{lightgray}49.950      &       & -     & \cellcolor{graywhite}27.092  & \cellcolor{lightgray}32.839         &                 & 51.627   & \cellcolor{graywhite}52.023  & \cellcolor{lightgray}52.329          \\

    40  & 45.410    & \cellcolor{graywhite}49.440      & \cellcolor{lightgray}49.720           &  & 34.250    & \cellcolor{graywhite}50.040      & \cellcolor{lightgray}51.000      &       & -     & \cellcolor{graywhite}25.347  & \cellcolor{lightgray}41.489          &                 & 51.937   & \cellcolor{graywhite}52.674  & \cellcolor{lightgray}52.842         \\

    &   &  \cellcolor{graywhite}    & \cellcolor{lightgray}\textcolor{red}{(+0.280)}   & &   & \cellcolor{graywhite}    & \cellcolor{lightgray}\textcolor{red}{(+0.960)} &  &   -   &   \cellcolor{graywhite}    & \cellcolor{lightgray}\textcolor{red}{(+16.142)}   &   &   & \cellcolor{graywhite}    & \cellcolor{lightgray}\textcolor{red}{(+0.168)}   \\
    
    % 50  & 46.730    & \cellcolor{graywhite}49.340      & \cellcolor{lightgray}49.430           &  & 34.630    & \cellcolor{graywhite}50.070      & \cellcolor{lightgray}51.120        &       & -     & \cellcolor{graywhite}28.045  & \cellcolor{lightgray}44.707          &                 & \cellcolor{graywhite}52.613   & 52.755  & \cellcolor{lightgray}52.902       \\

    60  & 43.350    & \cellcolor{graywhite}48.980      & \cellcolor{lightgray}49.110           &  & 34.300    & \cellcolor{graywhite}49.790      & \cellcolor{lightgray}50.490     &       & -     & \cellcolor{graywhite}35.788  & \cellcolor{lightgray}46.295     &       & 52.674   & \cellcolor{graywhite}52.898  & \cellcolor{lightgray}52.932       \\
    
    &   &  \cellcolor{graywhite}    & \cellcolor{lightgray}\textcolor{red}{(+0.130)}   & &   & \cellcolor{graywhite}    & \cellcolor{lightgray}\textcolor{red}{(+0.700)} &  &   -   &   \cellcolor{graywhite}    & \cellcolor{lightgray}\textcolor{red}{(+10.507)}   &   &   & \cellcolor{graywhite}    & \cellcolor{lightgray}\textcolor{red}{(+0.034)}   \\

    % 70  & 43.640    & \cellcolor{graywhite}48.910      & \cellcolor{lightgray}49.920           &  & 33.560    & \cellcolor{graywhite}49.620      & \cellcolor{lightgray}50.800       &       & -     & \cellcolor{graywhite}43.981  & \cellcolor{lightgray}50.823            &                 & 52.889   & \cellcolor{lightgray}52.971  & \cellcolor{graywhite}52.928      \\

    80  & 43.780    & \cellcolor{graywhite}49.370      & \cellcolor{lightgray}49.510 &     & 35.980    & \cellcolor{graywhite}50.190      & \cellcolor{lightgray}52.590       &       & -     & \cellcolor{graywhite}40.962  & \cellcolor{lightgray}53.575              &                 & \cellcolor{lightgray}52.850   & \cellcolor{graywhite}52.648  & 52.609    \\
    
    &   &  \cellcolor{graywhite}    & \cellcolor{lightgray}\textcolor{red}{(+0.140)}   & &   & \cellcolor{graywhite}    & \cellcolor{lightgray}\textcolor{red}{(+2.400)} &  &   -   &   \cellcolor{graywhite}    & \cellcolor{lightgray}\textcolor{red}{(+12.613)}   &   & \cellcolor{lightgray}  & \cellcolor{graywhite}    &    \\

    % 90  & 48.420    & \cellcolor{graywhite}49.600      & \cellcolor{lightgray}49.720           &   & 35.560    & \cellcolor{graywhite}50.890      & \cellcolor{lightgray}53.150      &       & -     & \cellcolor{graywhite}41.188  & \cellcolor{lightgray}56.283             &                 & 52.238   & \cellcolor{graywhite}52.557  & \cellcolor{lightgray}52.928      \\

    100 & 47.850    & \cellcolor{graywhite}49.660      & \cellcolor{lightgray}50.030           &  & 35.520    & \cellcolor{graywhite}51.310      & \cellcolor{lightgray}52.740       &       & -     & \cellcolor{graywhite}48.093  & \cellcolor{lightgray}59.226           &                 & 52.454   & \cellcolor{graywhite}52.100  & \cellcolor{lightgray}52.794       \\
    
    &   &  \cellcolor{graywhite}    & \cellcolor{lightgray}\textcolor{red}{(+0.370)}   & &   & \cellcolor{graywhite}    & \cellcolor{lightgray}\textcolor{red}{(+1.430)} &  &   -   &   \cellcolor{graywhite}    & \cellcolor{lightgray}\textcolor{red}{(+11.133)}   &   &   & \cellcolor{graywhite}    & \cellcolor{lightgray}\textcolor{red}{(+0.694)}   \\

    \bottomrule[2pt]
    \end{tabular}
}    
\end{table}

\textbf{Baselines.} We compare our FL-tuning (FL) with fine-tuning (FT), P-tuning v1 (PV1)~\cite{liu2021gpt}, and P-tuning v2 (PV2)~\cite{liu2021p}. The prompt length in both PV1 and PV2 are set to 160, and the number of added hidden units of each FFN layer in FL is also pre-defined as 160. FT's results are obtained by tuning all the Transformer's parameters without prompts. Results of PV1, PV2, and FL are obtained by freezing the original parameters of the Transformer and only tuning the introduced trainable parameters.

\textbf{Datasets.} 
We use NLU text datasets from the CLUE\footnote{https://www.cluebenchmarks.com/}~\cite{DBLP:conf/coling/XuHZLCLXSYYTDLS20} benchmark as experimental data, which is similar to GLUE~\cite{DBLP:conf/iclr/WangSMHLB19} and SuperGLUE~\cite{DBLP:conf/nips/WangPNSMHLB19}. We select 11 NLU datasets from CLUE for the comparison experiments. The scale of each dataset is shown in Table \ref{tab:hyperparameters}. The selected 11 text datasets are divided into 7 kinds of NLU tasks, including Text Classification (TC), Pronoun Disambiguation (PD), Semantic Similarity (SS), Natural Language Inference (NLI), Keyword Recognition (KR), Name Entity Recognition (NER), and Machine Reading Comprehension (MRC).

\textbf{Evaluation metrics.}
We report F1 (\%) as the evaluation metric for the NER task and Accuracy (\%) for other NLU tasks. All comparison results are obtained after submitting to CLUE and the detailed analysis is based on the parameters of the submitted models.

% Evaluation metric for NER task is F1 (\%) and other NLU tasks's evaluation metrics are Accuracy (\%).

\subsection{Comparison Results Across Tasks}

Similar to CLUE~\cite{DBLP:conf/coling/XuHZLCLXSYYTDLS20}, we divide the selected seven kinds of NLU tasks into four categories: Single-Sentence Tasks, Sentence Pair Tasks, Name Entity Recognition Tasks, and Machine Reading Comprehension Tasks.

% 1. 
% 1. 在大多数情况下，A比B和C好，验证lay tuning在任务A上是有效的。
% 2. 特别，在PD任务上，基于NEZHA模型的FL比PV2高了将近18\%on C3 1.0数据集。
% 3. 基于A模型的FL比FT高了16在数据集上。
% 4. 模型之所以在PD任务上提升明显的原因可能是它的训练数据集比较少，而FL比其他方法更适用于小样本场景中。
% 5. 为了验证这一论点，我们在A中补充了实验。

\textbf{Single-Sentence Tasks} include TC and PD.  We adopt the IFLYTEK dataset for long TC, TNEWS for short TC, and WSC for PD. The experimental results on TC and PD tasks are shown in Table \ref{tab:TC} and Table \ref{tab:PDandNER}, respectively. From the table, we observe that FL-tuning outperforms fine-tuning and prompt tuning in most cases, indicating the effectiveness of our method on TC and PD tasks. In particular, the improvement of NEZHA-based FL-tuning on PD task is very obvious, and its accuracy is 12.76\% and 17.93\% higher than that of fine-tuning and P-tuning v2. The reason may be that the training dataset of WSC is relatively small and FL-tuning is more suitable for this scenario than other methods. To verify this argument, we supplement the experiments under few-shot settings later in this subsection.

% : 1) FL-tuning achieves the best performance on most cases of these two tasks. 2) In particular, the improvement of FL-tuning on PD tasks is very obvious, with a maximum increase of nearly 18\%. \textcolor{blue}{The reason may be that FFN learns more abstract features and strengthens the representation by mapping the data to a high-dimensional space and then back to a low-dimensional space. FL-tuning increases the dimension of the high-dimensional space, so that the output of FFN learns more abstract features and incorporates more activation memory distributions, which in turn enhances the performance.}

% 1. 在绝大多数情况下，B比C好。特别地，高了2\%
% 2. 特别地，我们的方案与几乎达到了方案B的水准(平均差距不到1\%)，但是我们的方案所使用的参数仅仅是方案B参数的3/1000，极大的降低了模型的训练成本。

\textbf{Sentence Pair Tasks} aim to predict relations between sentence pairs, or abstract-keyword pairs, including SS (AFQMC), NLI (CMNLI and OCNLI), and KR (CSL). Table \ref{tab:SSandNLIandKR} shows the results on SS, NLI, and KR. According to the results, we conclude that FL-tuning is better than prompt tuning in almost all cases. In particular, it achieves a 2.15\% improvement over 78.77\% obtained by RoFormer-based PV2. Although FT's results are slightly better than ours, the number of learnable parameters in FL-tuning is only about 3\% of Transformer's parameters, which greatly reduces the training cost.

% : 1) Fine-tuning performs best in most cases. \textcolor{blue}{When training downstream tasks, Fine-tuning updates all pre-training parameters, so it is reasonable that the fitting effect is best.} 2) Although FL-tuning performs the best in only a few cases, it is better than PV2 in most cases. 

\textbf{Name Entity Recognition Tasks} experiment on the CLUE Fine-Grain NER (CLUENER) dataset and the results are shown in Table \ref{tab:PDandNER}. 
From the table, we observe that FL-tuning performs best among tuning methods in all cases, which further demonstrates the effectiveness of our PLMs' tuning method.

% CLUENER is based on the open source text classification dataset THUCNews, with some data selected for fine-grained named entity annotation, originally from Sina News RSS. \textcolor{blue}{From the table, FL-tuning performs best on all PLMs, especially for RoFormer.}

\textbf{Machine Reading Comprehension Tasks} include C3, CHID, and CMRC2018 datasets. The results are shown in Table \ref{tab:RC}. According to the results, we find that the performance of FL-tuning on MRC tasks is similar to that of Sentence Pair tasks. Although it is not as good as fine-tuning in most cases, it significantly outperforms prompt tuning methods. This verifies that layer tuning is more effective than prompt tuning.

% Although not as good as Fine-tuning in most cases, it outperforms significantly PV2 in most cases. \textcolor{blue}{This verifies the effectiveness of our proposed Layer Prompt.}

% 在四类任务中各选择了一个数据集进行小样本对比实验
% 各数据集的实验参数选用其全量数据集在RoBERTa上效果最好的参数
% 设定数据集样本的数量从10-100的整十进行试验
% 实验结果如表所示
\textbf{Few-shot Setting.}
In addition to the comparison on the full data, we use RoBERTa as the backbone to compare FL-tuning with prompt tuning methods under few-shot setting. We set the training set to be $\{20, 40, 60, 80, 100\}$. The experimental results on TNEWS 1.0, CMNLI, CLUENER, and CHID datasets are reported in Table \ref{tab:Few-shot}. It can be found that FL-tuning is better than PV1 and PV2 in almost all cases, indicating that our method is also suitable for few-shot learning. In particular, our FL-tuning improves F1 by more than 10\% over PV2 in almost all training sample size settings on the CLUENER dataset.

\begin{figure}[t]
\centering
\resizebox{\textwidth}{!}
{
    \subfigure[WSC]{
        \begin{minipage}[t]{0.25\linewidth}
        \centering
        \includegraphics[width=1\textwidth]{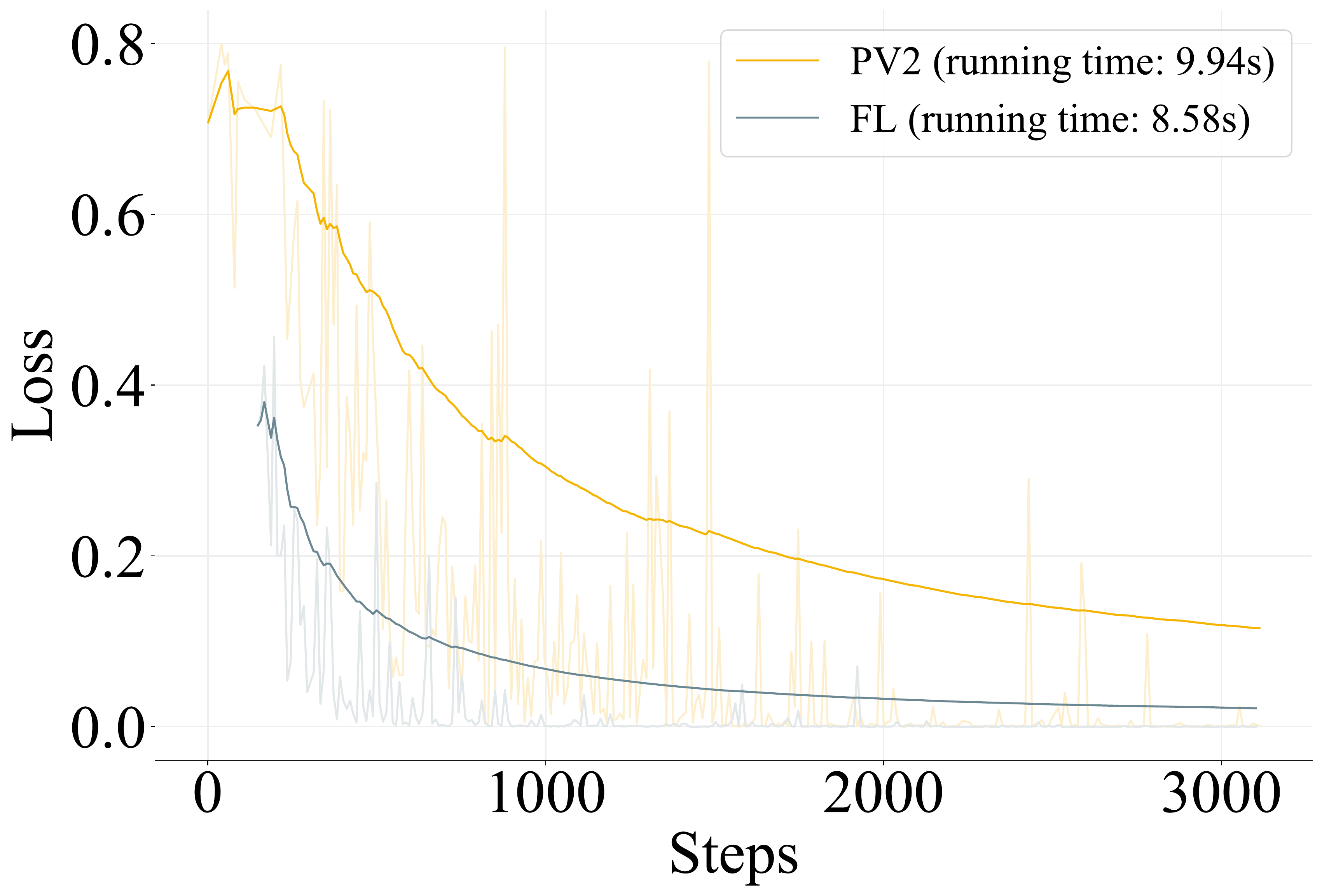}
        \label{fig:loss-WSC}
        \end{minipage}
    }
    \subfigure[OCNLI]{
        \begin{minipage}[t]{0.25\linewidth}
        \centering
        \includegraphics[width=1\textwidth]{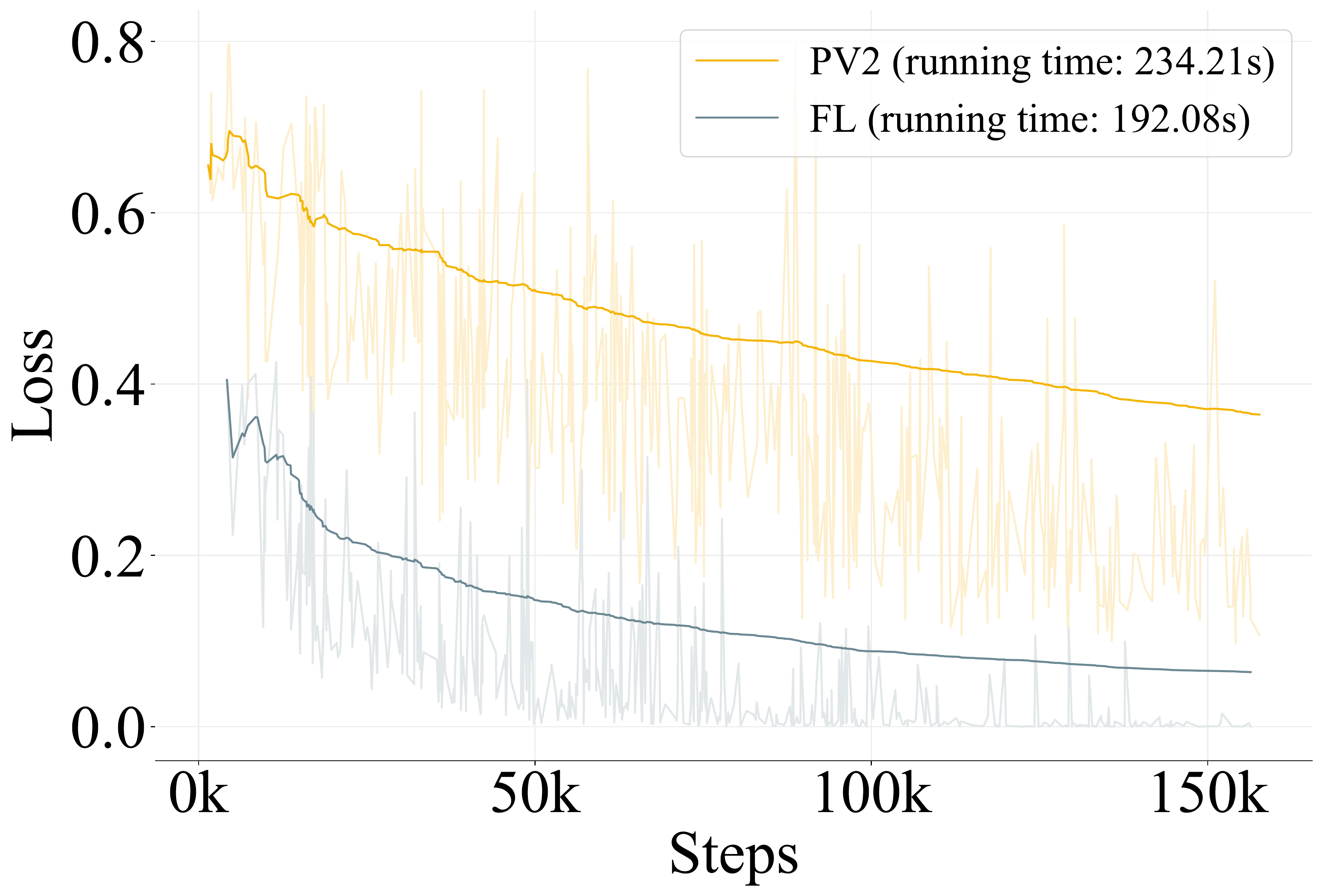}
        \label{fig:loss-OCNLI}
        \end{minipage}
    }
    \subfigure[CLUENER]{
        \begin{minipage}[t]{0.25\linewidth}
        \centering
        \includegraphics[width=1\textwidth]{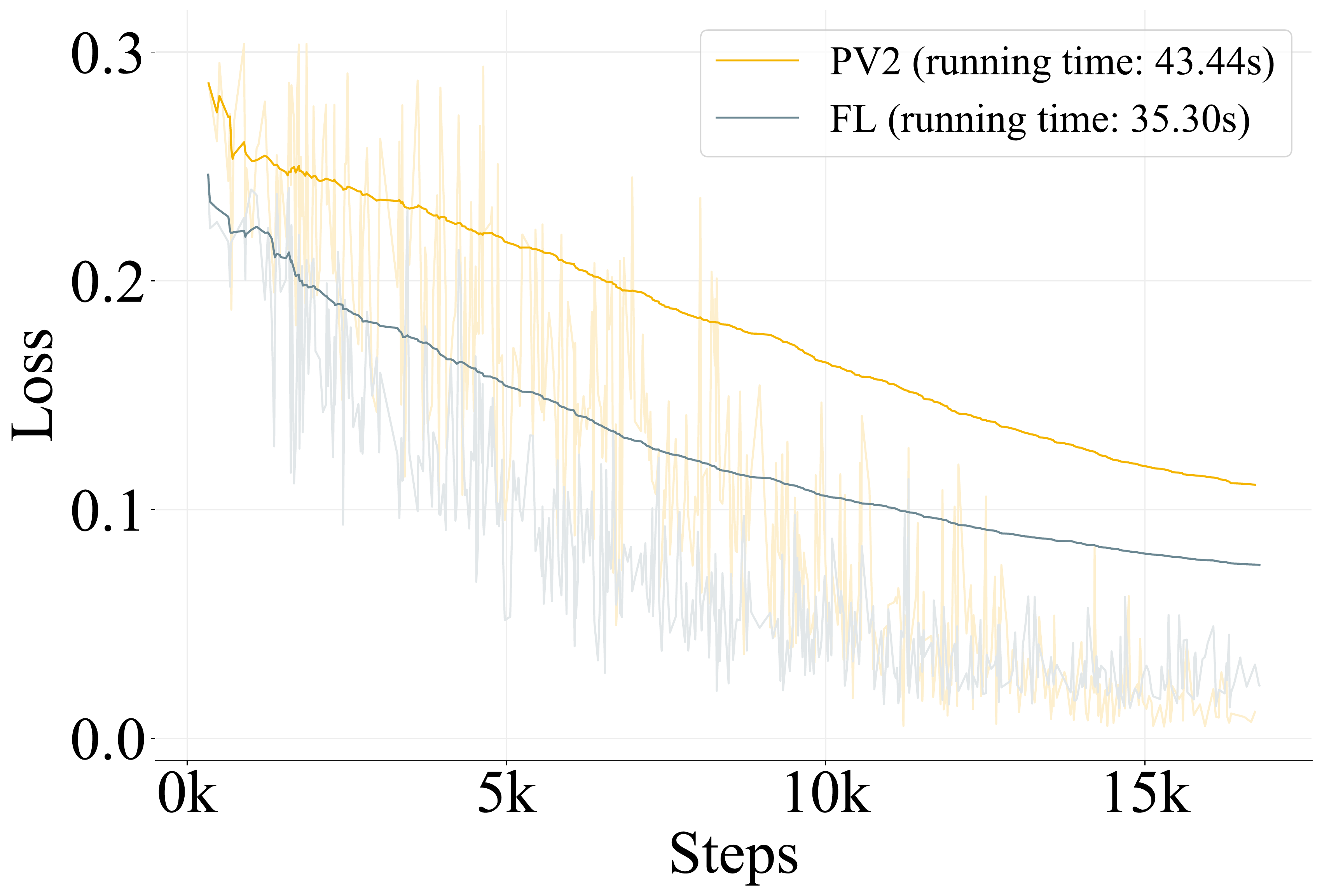}
        \label{fig:loss-CLUENER}
        \end{minipage}
    }
    \subfigure[CMRC2018]{
        \begin{minipage}[t]{0.25\linewidth}
        \centering
        \includegraphics[width=1\textwidth]{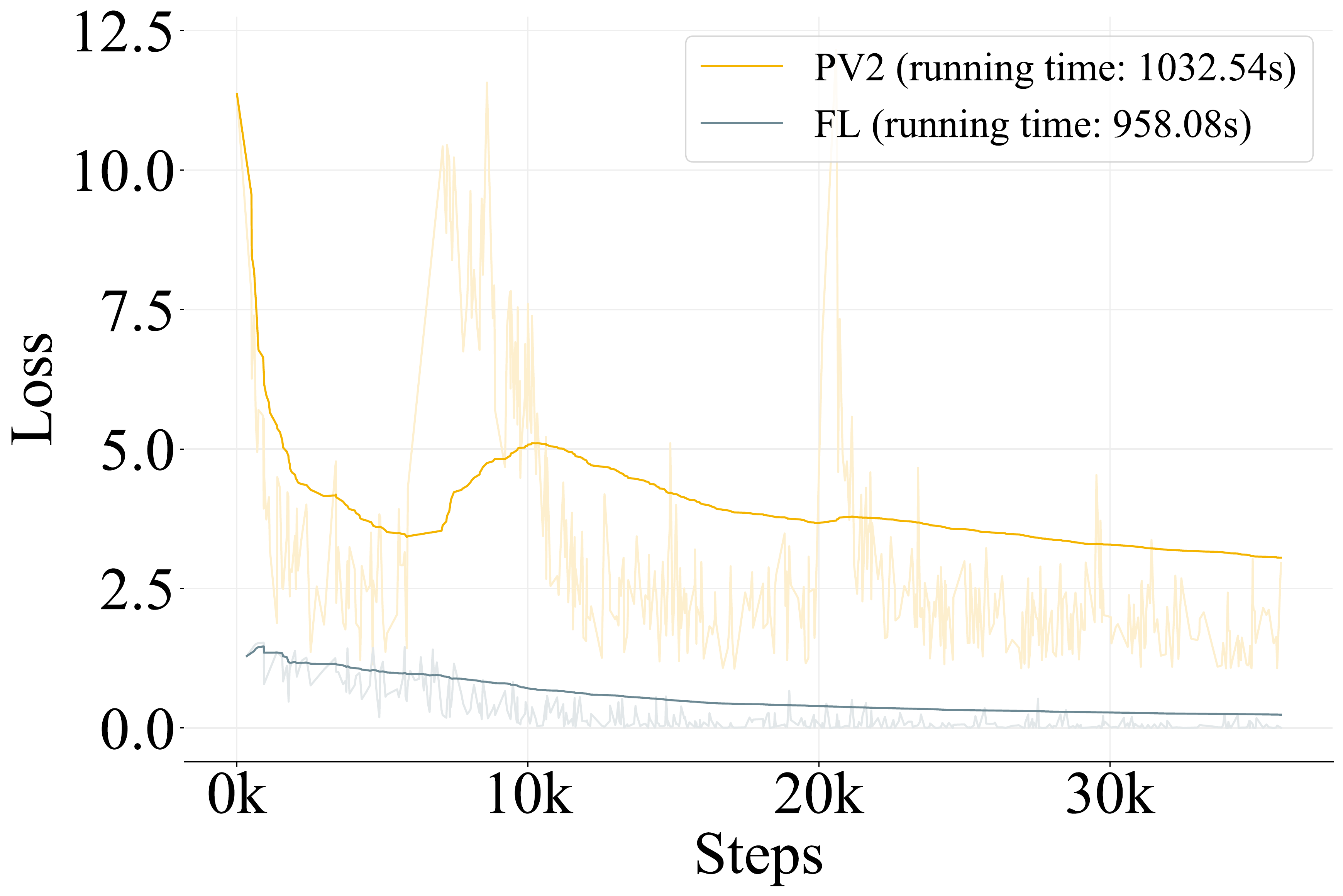}
        \label{fig:loss-CMRC}
        \end{minipage}
    }
}
\caption{Convergence comparison of FL-tuning and P-tuning v2. The curves in dark color are obtained by smoothing the loss curves (light color). The smoothing function is $\alpha * previousStep\_smoothed\_value +(1-\alpha)*current\_value$, where $\alpha=0.99$ is the smooth weight. The models run on Ubuntu 20.04 with Intel Xeon Gold 6248R*2, NVIDIA V100 and 512G of RAM. FL-tuning is more stable and converges faster than PV2.}
% Comparison of convergence between FL-tuning and P-tuning v2. Yellow represents P-tuning v2 and blue represents FL-tuning. The parameter of the horizontal axis is step and the vertical axis is loss. The smooth calculation principle is: $smooth\_weight*last\_smoothed\_value$ $+(1-smooth\_weight)*current\_value$, and the $smooth\_weight$ is 0.99.}
\label{fig:loss}
\end{figure}

\begin{figure}[t]
\centering
\resizebox{\textwidth}{!}
{
    \subfigure[WSC (P-tuning v2)]{
        \begin{minipage}[t]{0.25\linewidth}
        \centering
        \includegraphics[width=1\textwidth]{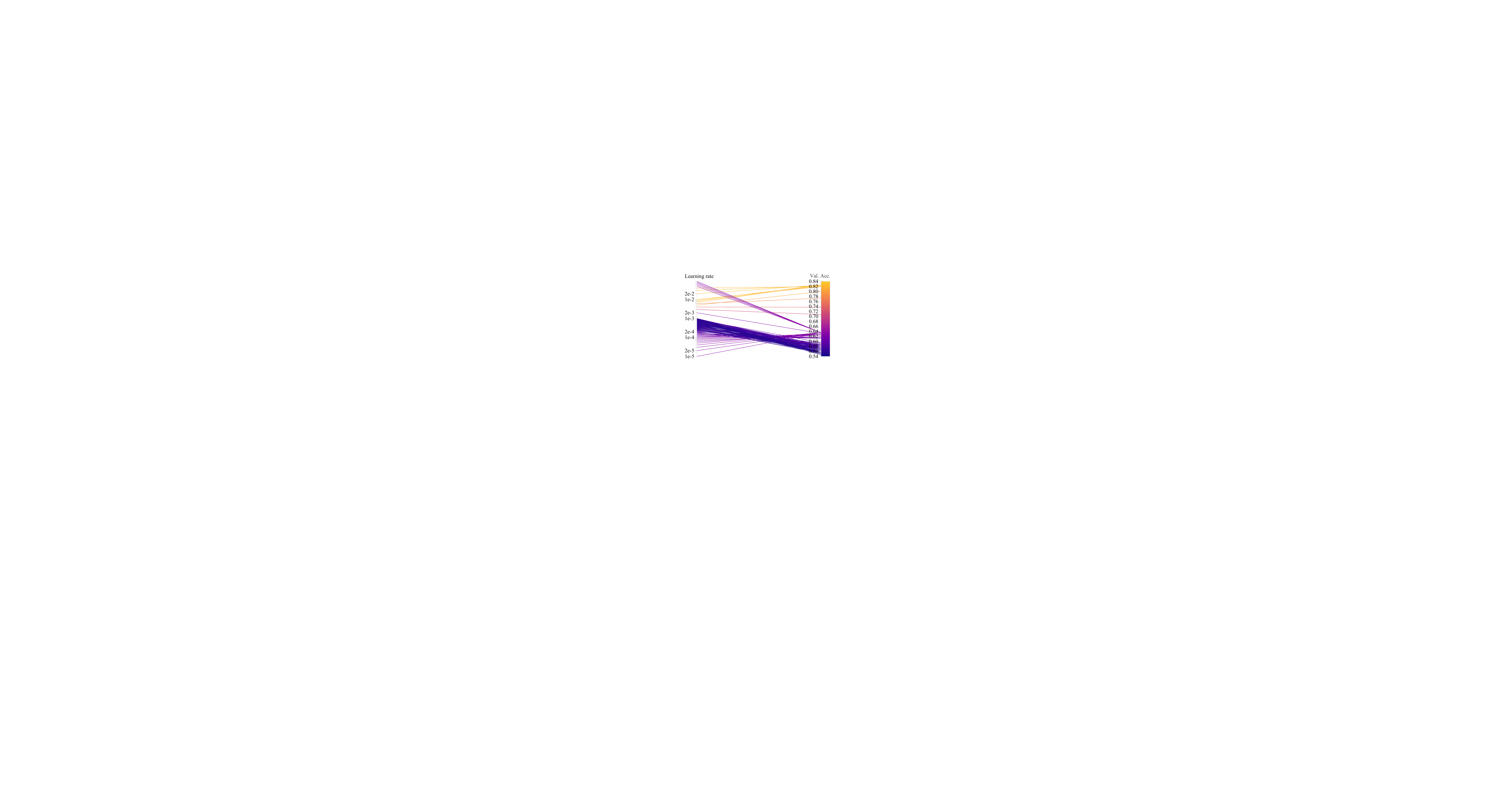}
        \label{fig:hyper-PV2-WSC}
        \end{minipage}
    }
    \subfigure[WSC (FL-tuning)]{
        \begin{minipage}[t]{0.25\linewidth}
        \centering
        \includegraphics[width=1\textwidth]{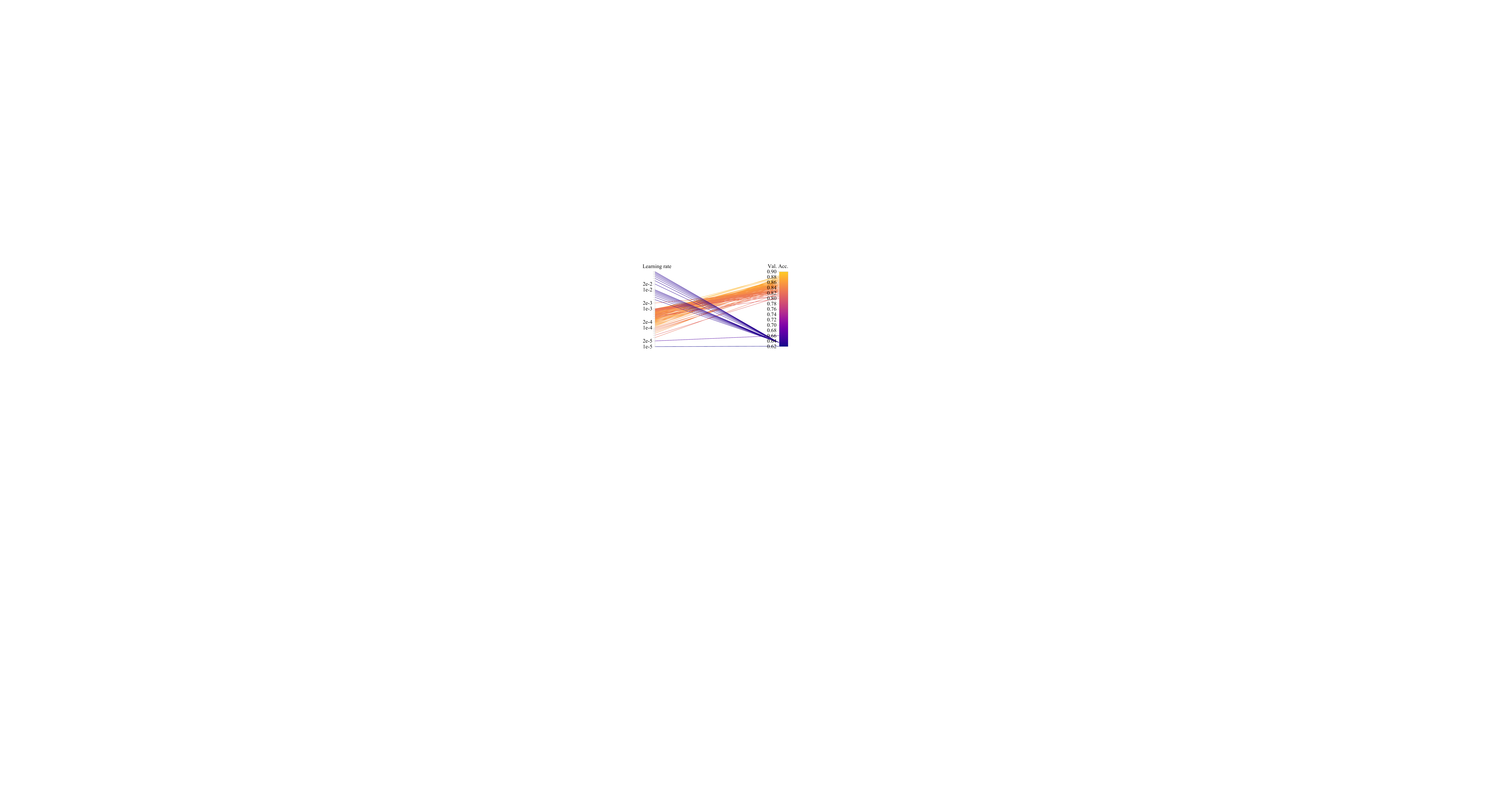}
        \label{fig:hyper-FL-WSC}
        \end{minipage}
    }
    \subfigure[CLUENER (P-tuning v2)]{
        \begin{minipage}[t]{0.25\linewidth}
        \centering
        \includegraphics[width=1\textwidth]{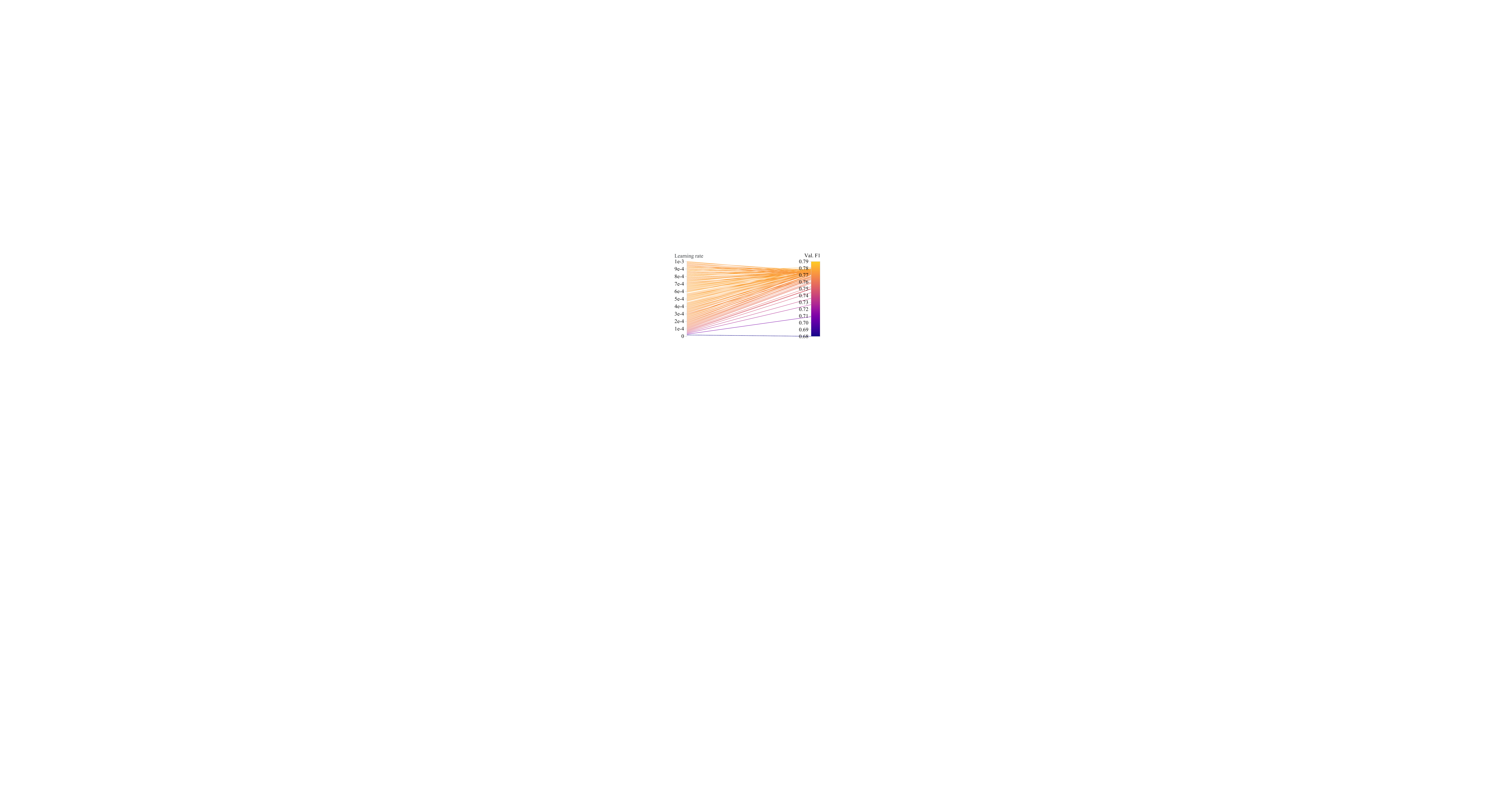}
        \label{fig:hyper-PV2-NER}
        \end{minipage}
    }
    \subfigure[CLUENER (FL-tuning)]{
        \begin{minipage}[t]{0.25\linewidth}
        \centering
        \includegraphics[width=1\textwidth]{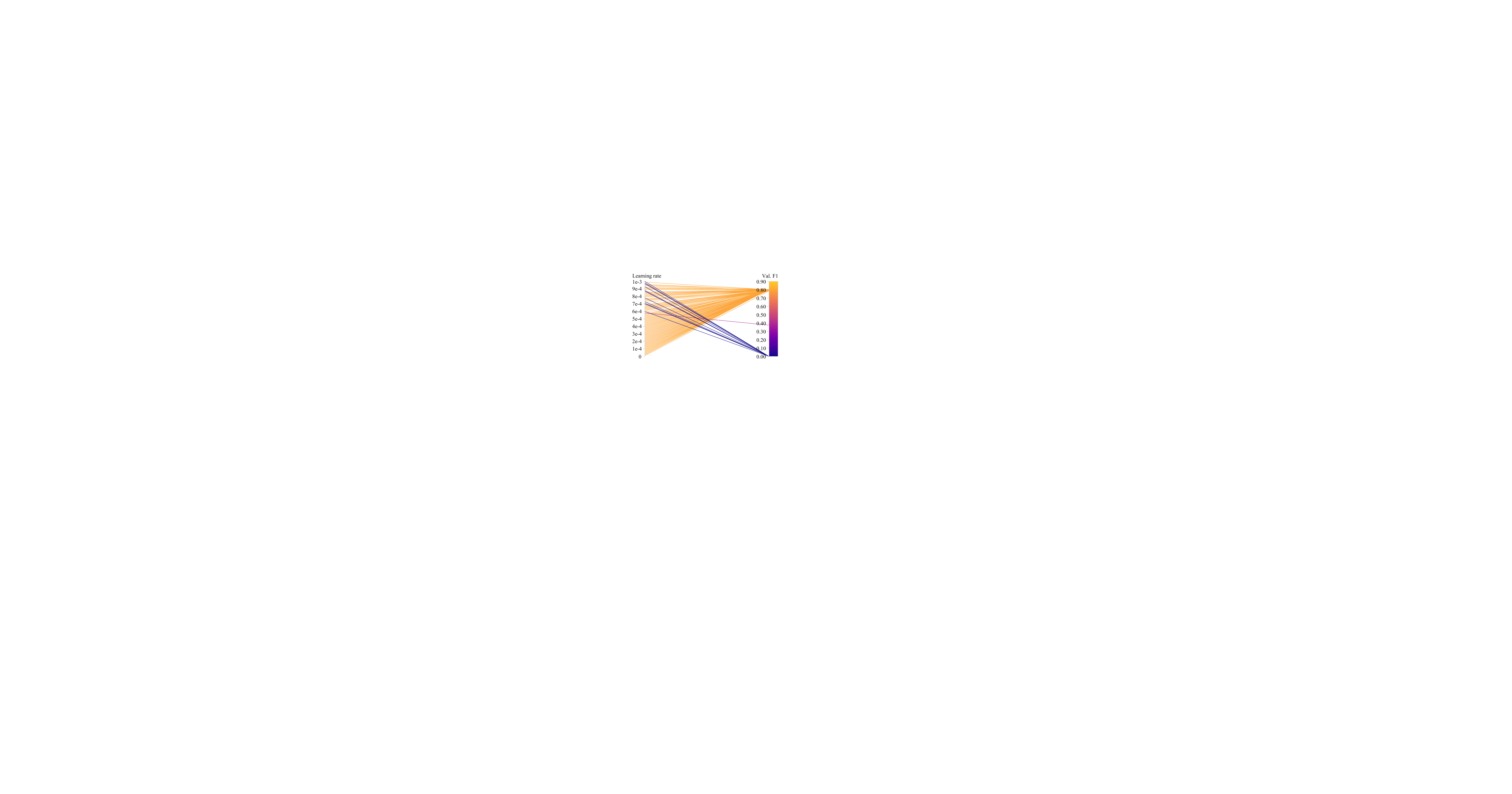}
        \label{fig:hyper-FL-NER}
        \end{minipage}
    }
}
\caption{Hyperparameter sensitivity of FL-tuning and P-tuning v2. The former is more insensitive to the learning rate than the latter.}
% Compare the sensitivity of FL-tuning and P-tuning v2 to hyperparameters based on RoBERTa. }
\label{fig:hyperparameter}
\end{figure}

\begin{figure}[!t]
\centering
\resizebox{\textwidth}{!}
{
    \subfigure[WSC]{
        \begin{minipage}[t]{0.25\linewidth}
        \centering
        \includegraphics[width=1\textwidth]{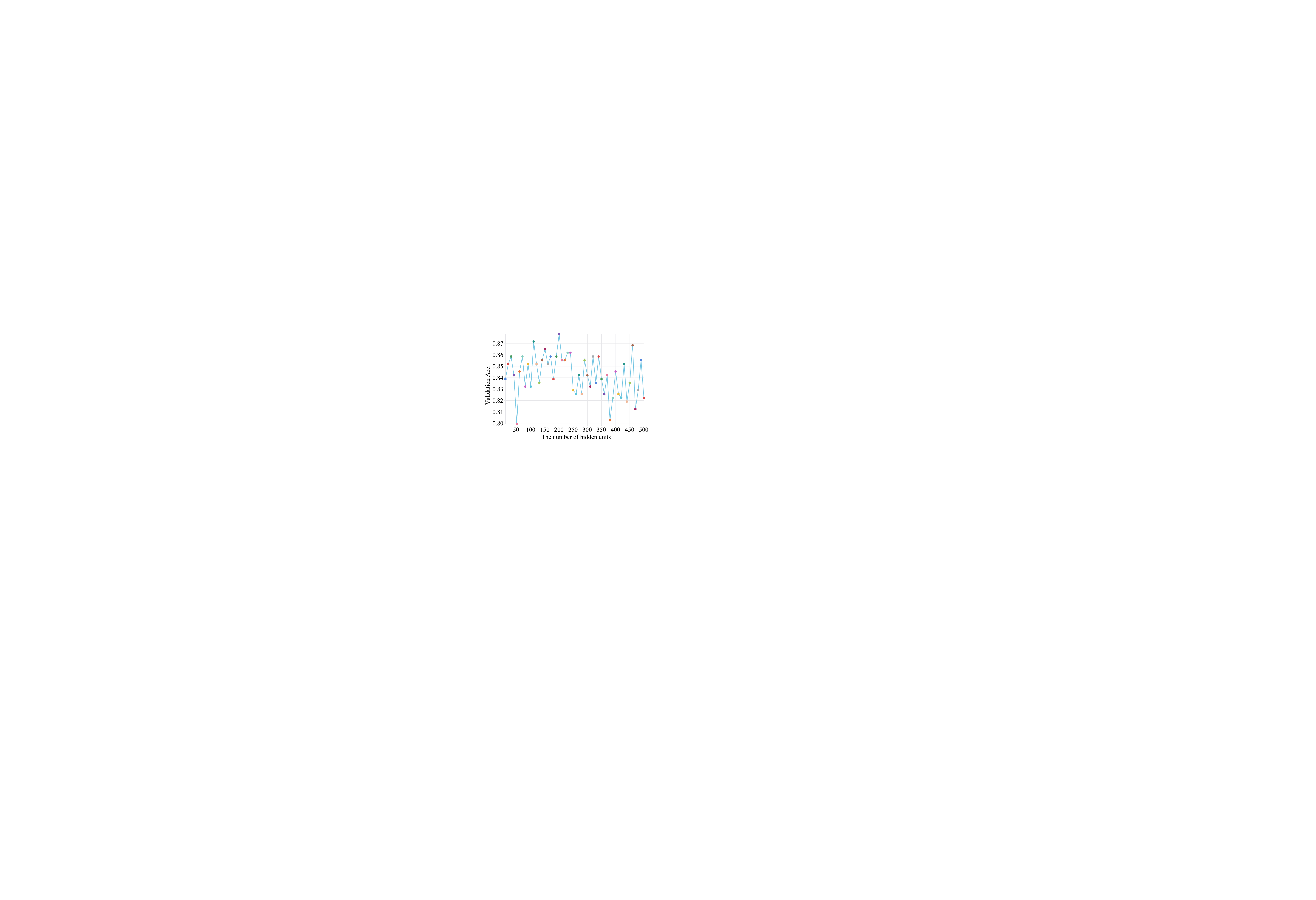}
        \label{fig:pl-wsc}
        \end{minipage}
    }
    \subfigure[OCNLI]{
        \begin{minipage}[t]{0.25\linewidth}
        \centering
        \includegraphics[width=1\textwidth]{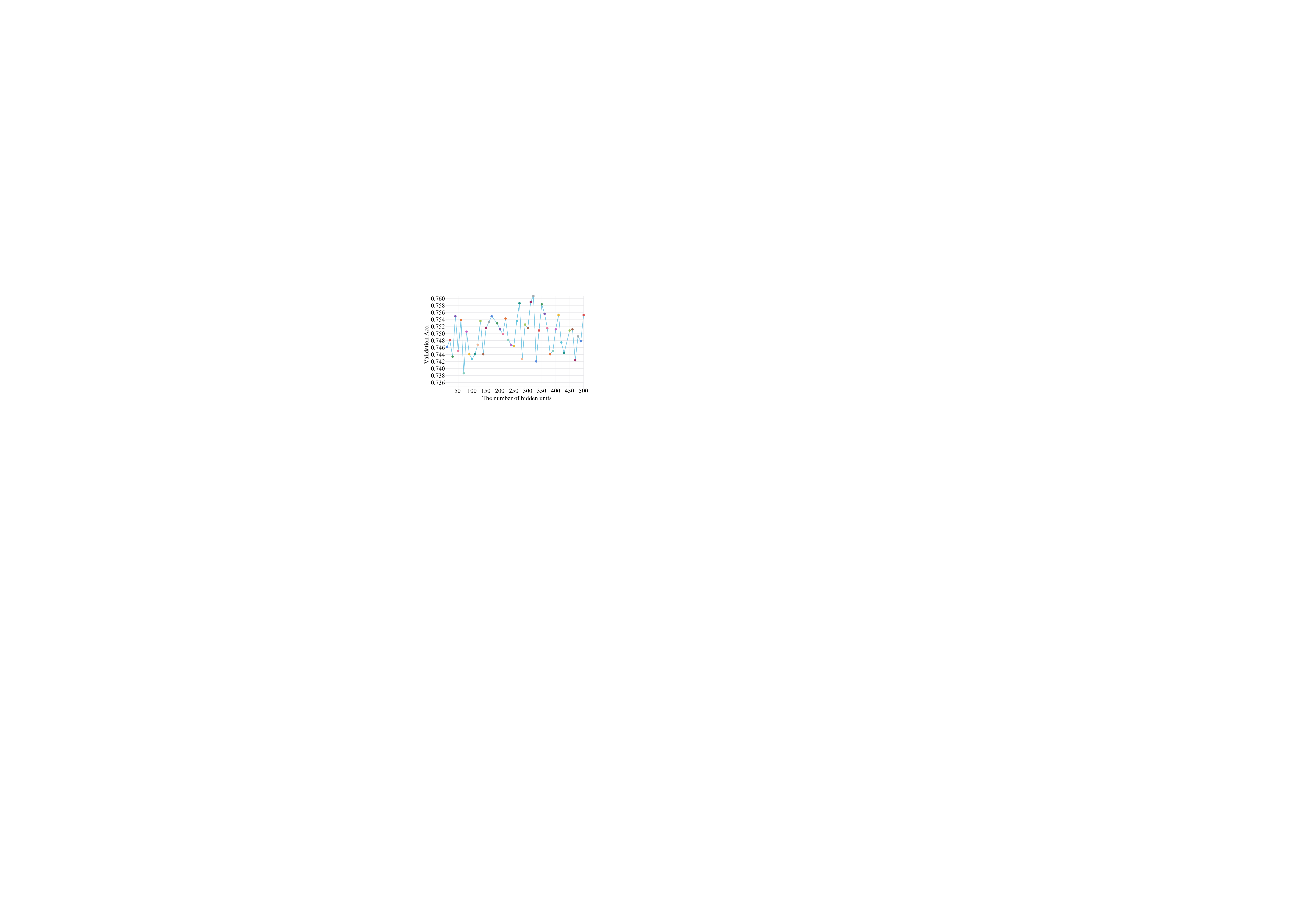}
        \label{fig:pl-ocnli}
        \end{minipage}
    }
    \subfigure[CLUENER]{
        \begin{minipage}[t]{0.25\linewidth}
        \centering
        \includegraphics[width=1\textwidth]{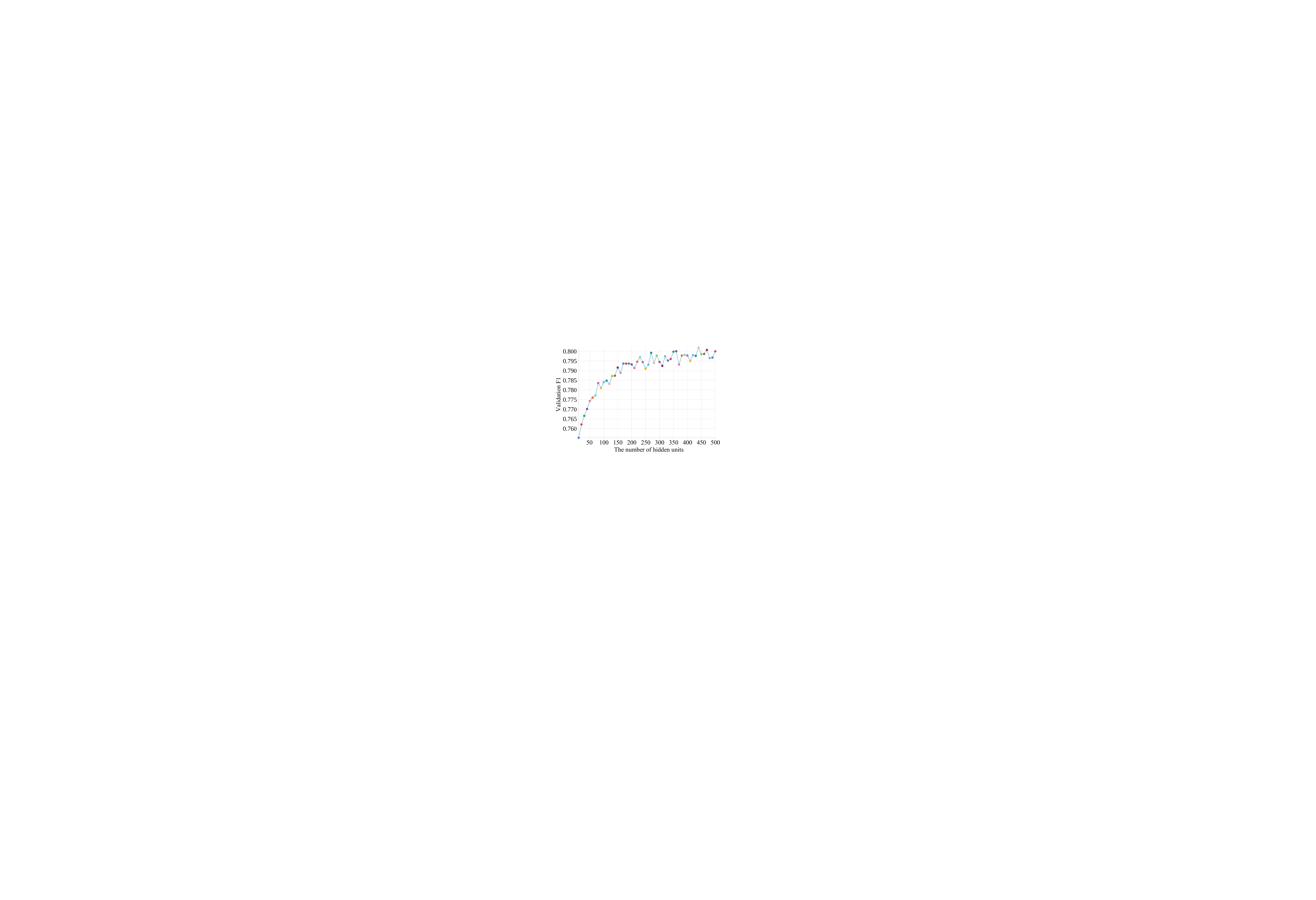}
        \label{fig:pl-cluener}
        \end{minipage}
    }
    \subfigure[CMRC2018]{
        \begin{minipage}[t]{0.25\linewidth}
        \centering
        \includegraphics[width=1\textwidth]{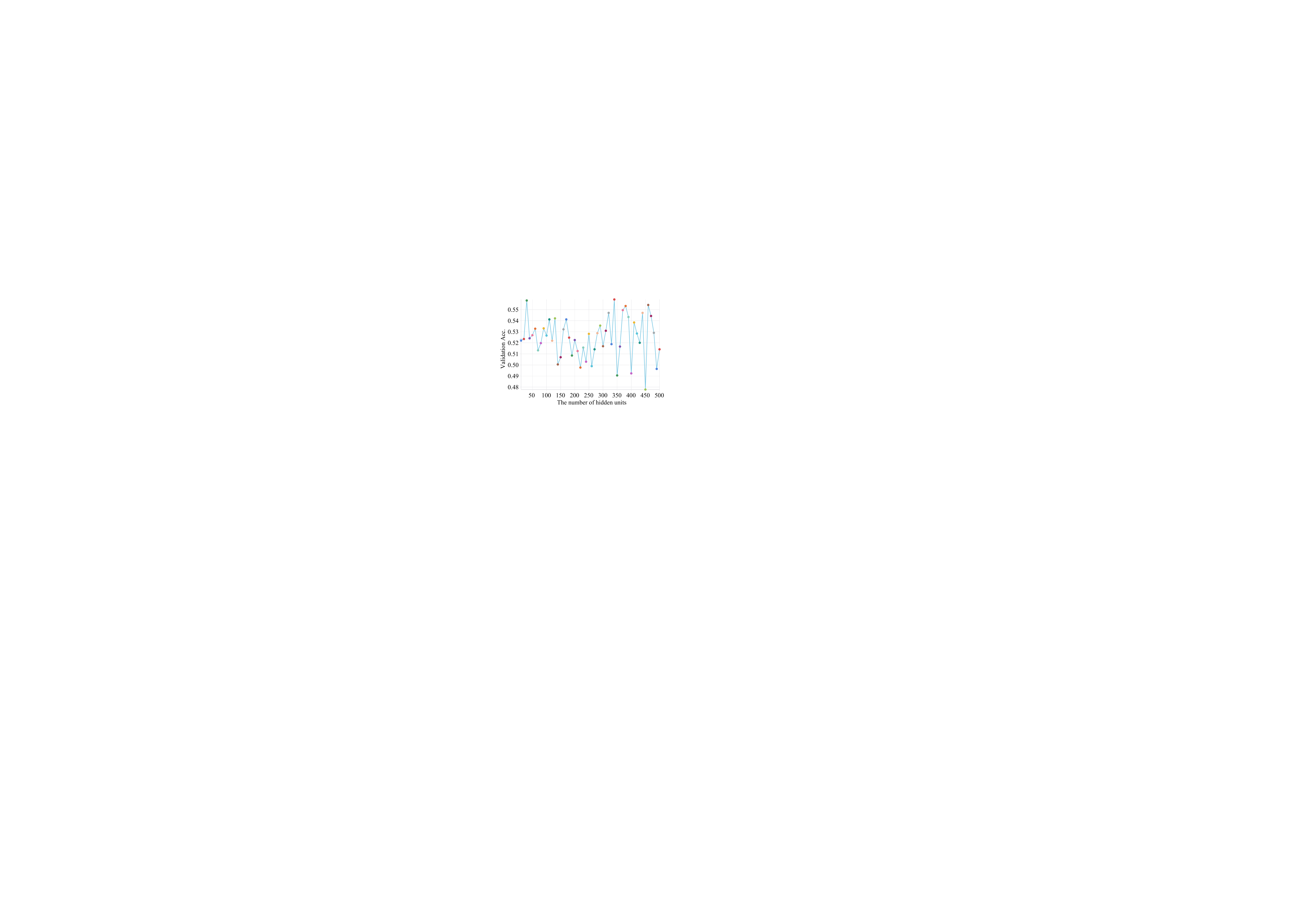}
        \label{fig:pl-CMRC2018}
        \end{minipage}
    }
}
\caption{The impact of the number of hidden units on model performance. ``[x-y]'' refers to the layer interval where we add hidden units. The optimal value of the number of hidden unit varies from task to task.}
% Ablation study on prompt depth. ``[x-y]'' refers to the layer interval where we add continuous prompts. Both prompt length are 160.}
\label{fig:plength}
\end{figure}

% Attention
\begin{table}[t]
\caption{Comparison results of FL-tuning and MA-tuning (MA) on RoBERTa. It is more effective to realize layer tuning on FFN than multi-head self-attention.
}
\label{tab:FLvsAP}
\centering
\resizebox{\textwidth}{!}
{
    \begin{tabular}{ccccccccccc}
    \toprule[2pt]
    & \textbf{IFLYTEK}    & \textbf{TNEWS 1.1}     & \textbf{WSC 1.1}  & \textbf{AFQMC}    & \textbf{CMNLI}  &\textbf{CSL} & \textbf{CLUENER} & \textbf{C3 1.0}    & \textbf{CHID}  & \textbf{CMRC2018}\\ 
    \cline{2-11}

    FL-tuning   & \cellcolor{lightgray}{62.000}    & \cellcolor{lightgray}{57.700}     & 81.240     & 72.990     & \cellcolor{lightgray}{80.920}     & \cellcolor{lightgray}{85.030}     & \cellcolor{lightgray}{80.839}    & \cellcolor{lightgray}{73.720}  & \cellcolor{lightgray}{86.953} & \cellcolor{lightgray}{72.050}   \\

    MA-tuning  & 61.540  & 57.140     & \cellcolor{lightgray}{81.590}     & \cellcolor{lightgray}{73.740}     & 80.570     & 84.400     & 80.319  & 73.170    & 85.868   & 71.850  \\

    \bottomrule[2pt]
    \end{tabular}
}
\end{table}

\begin{figure}[t]
\centering
\resizebox{\textwidth}{!}
{
    \subfigure[TNEWS]{
        \begin{minipage}[t]{0.25\linewidth}
        \centering
        \includegraphics[width=1\textwidth]{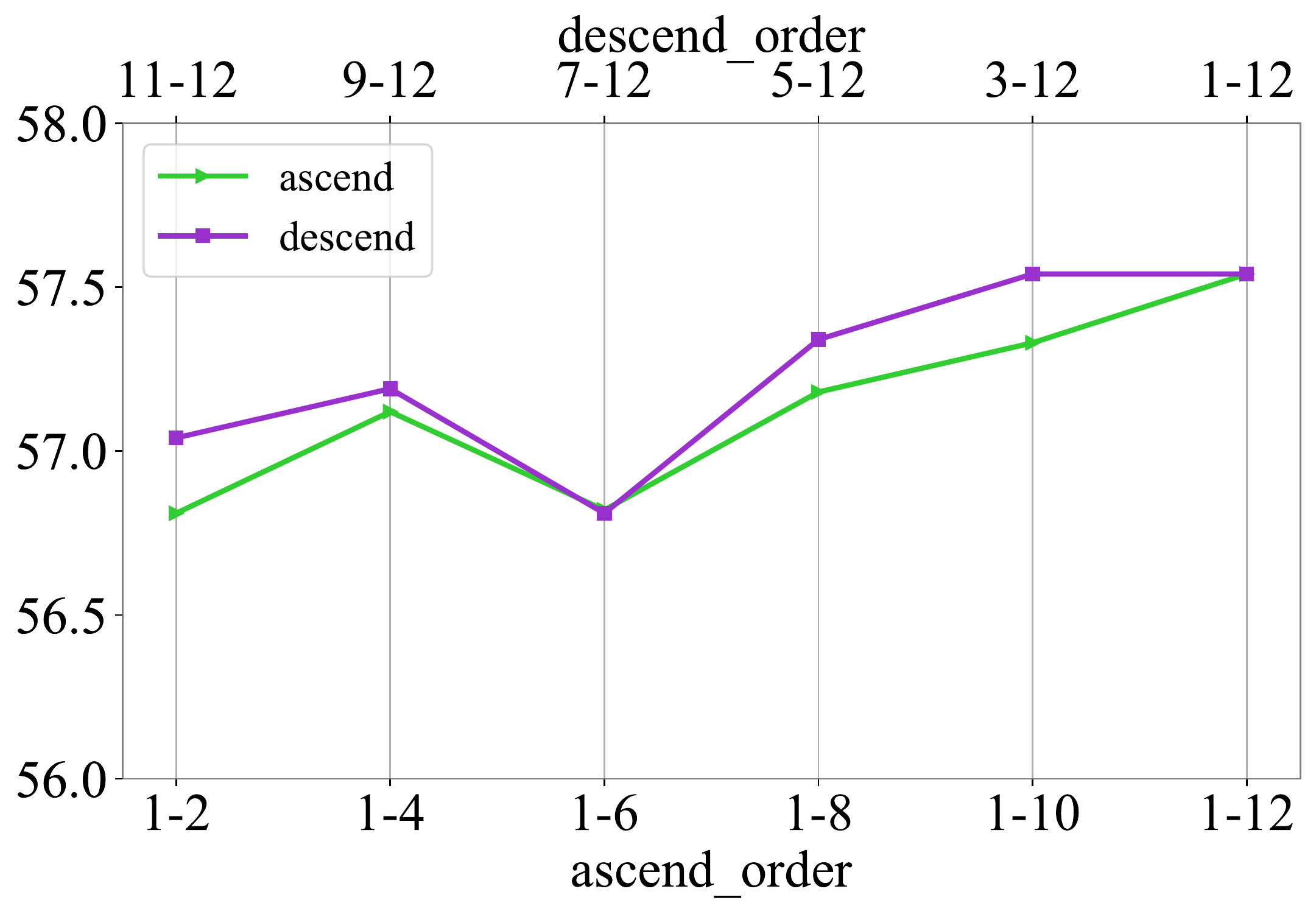}
        \label{fig:pd-tnews}
        \end{minipage}
    }
    \subfigure[WSC]{
        \begin{minipage}[t]{0.25\linewidth}
        \centering
        \includegraphics[width=1\textwidth]{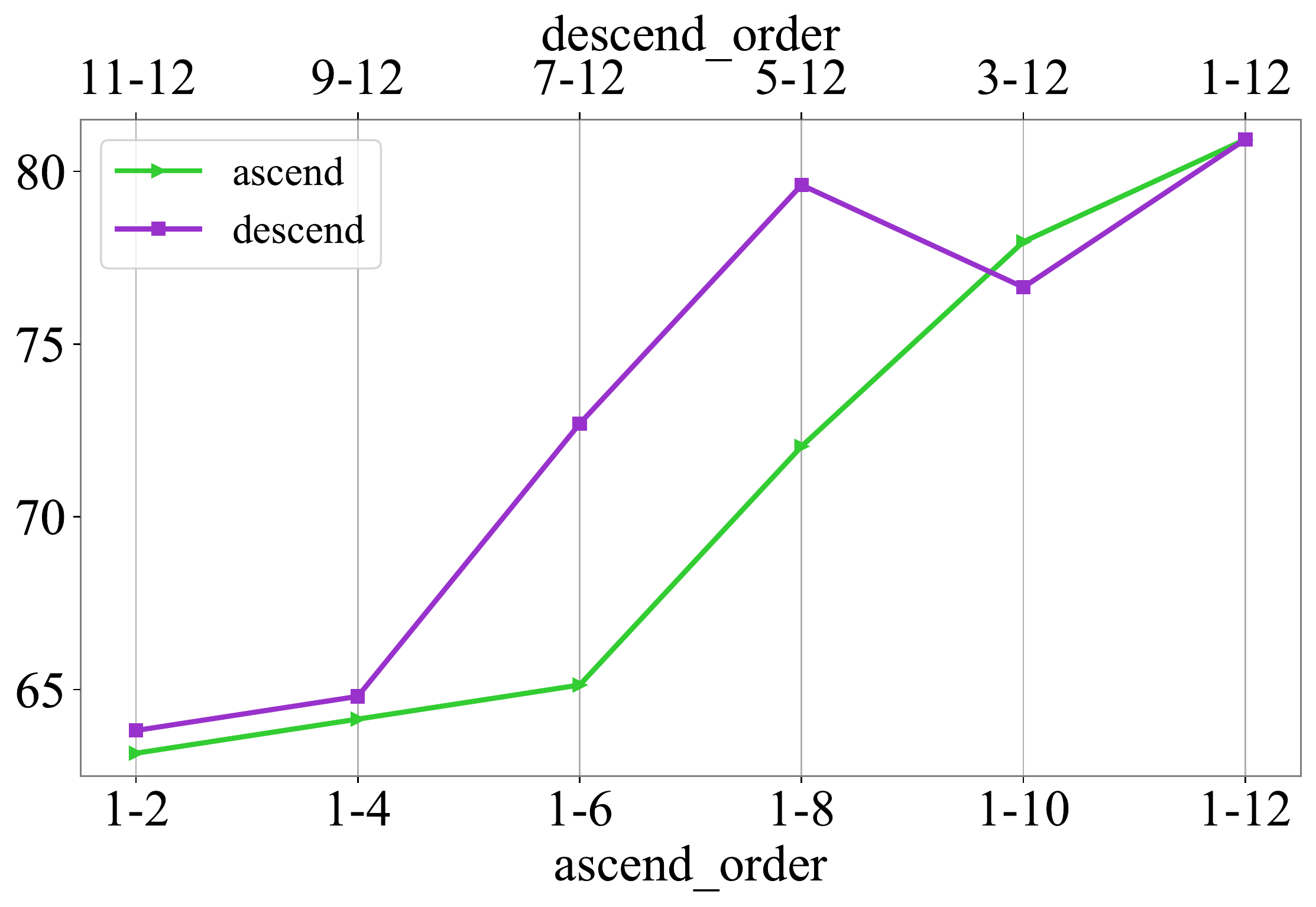}
        \label{fig:pd-WSC}
        \end{minipage}
    }
    \subfigure[OCNLI]{
        \begin{minipage}[t]{0.25\linewidth}
        \centering
        \includegraphics[width=1\textwidth]{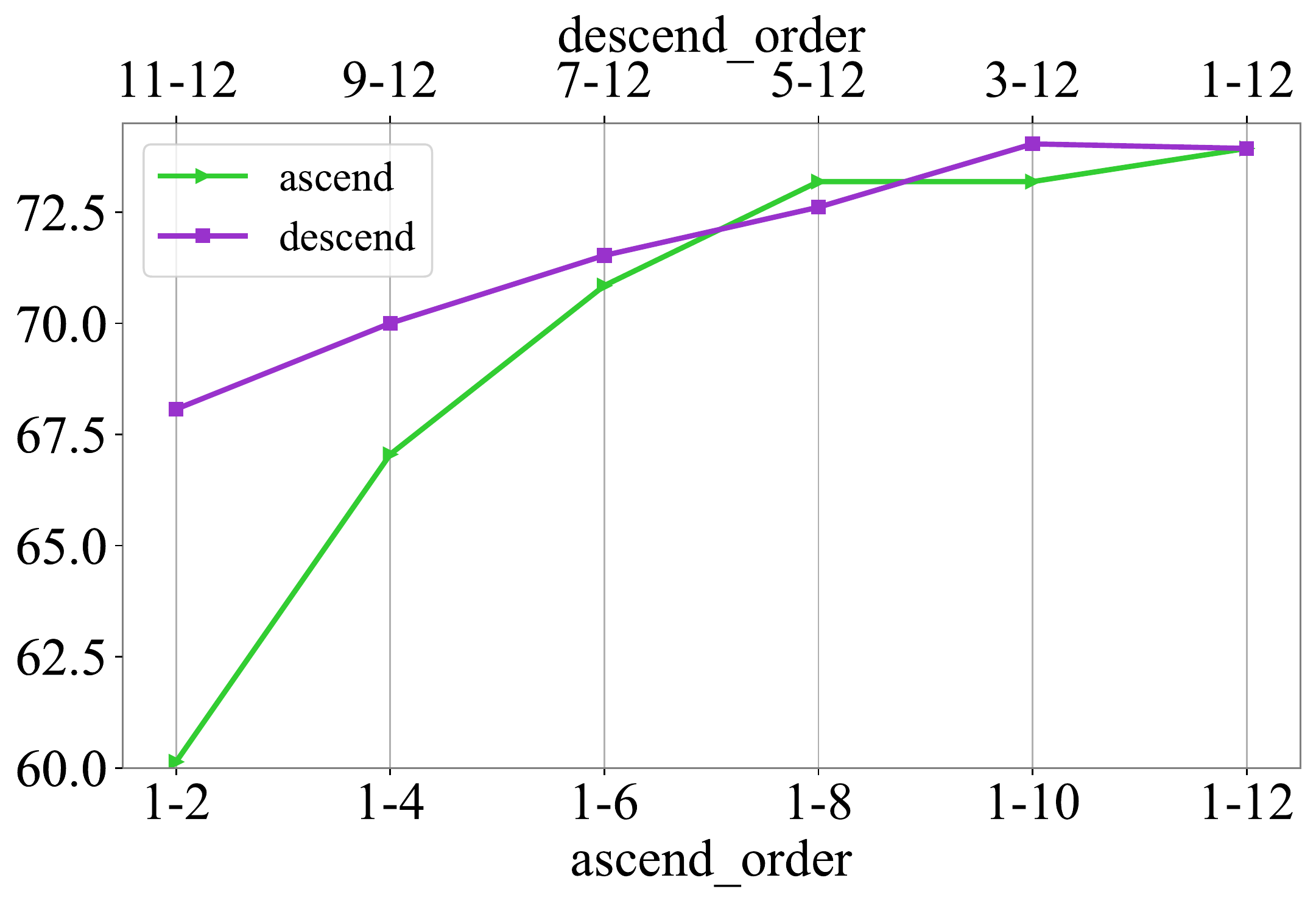}
        \label{fig:pd-OCNLI}
        \end{minipage}
    }
    \subfigure[CLUENER]{
        \begin{minipage}[t]{0.25\linewidth}
        \centering
        \includegraphics[width=1\textwidth]{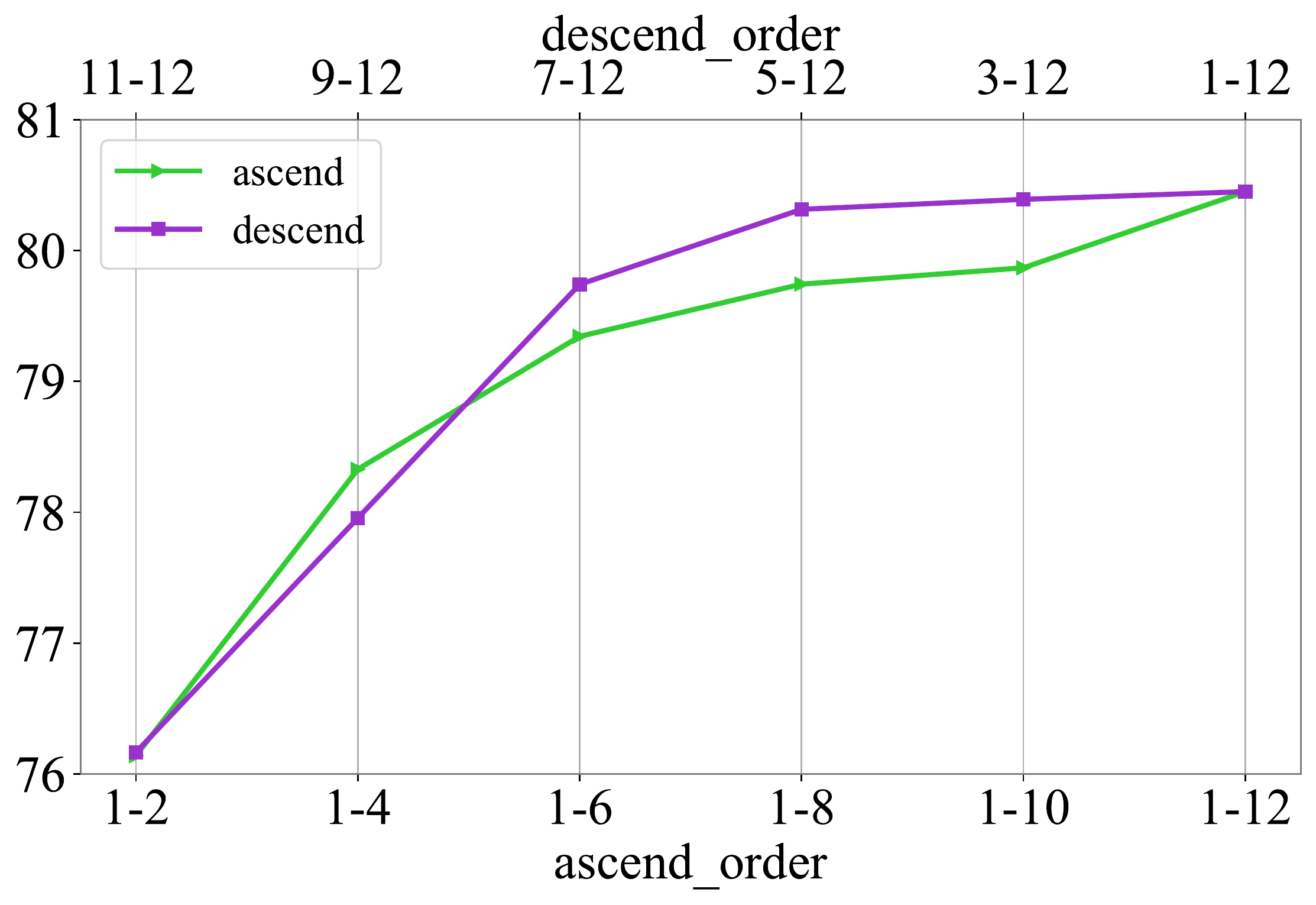}
        \label{fig:pd-CLUENER}
        \end{minipage}
    }
}
\caption{The impact of the depth of hidden units on model performance. The performance of our FL-tuning is positively correlated with the number of added layers. In addition, the deeper the layer of hidden units, the better the model performance.}
\label{fig:PromptDepth}
\end{figure}

\textbf{Summary:} Combing the above results, we conclude that: 1) The performance of FL-tuning is much better than that of prompt tuning methods. 2) With only about 3\% of Transformer's parameters to be trained, FL-tuning is comparable with fine-tuning on most datasets, and significantly outperforms fine-tuning on several datasets. 3) Compared with the full-data setting, FL-tuning's performance gain over prompt tuning methods is larger under the few-shot setting.

\subsection{Detailed Analysis}
% 1. 接下来，我们主要分析FL-tuning的收敛性、超参敏感性、以及neuron的个数和neuron 深度对模型的影响。
Next, we take RoBERTa as the backbone to analyze our FL-tuning in detail from the perspectives of convergence, hyperparameter sensitivity as well as the impact of the number and depth of added hidden units on model performance.

% \textbf{Converges faster in stable.}
\textbf{Faster convergence and more stable training.}
% 1. 收敛性。我们与PV2进行收敛速度与稳定性的对比。结果如图4所示。
% 2. 从图中，我们得知在训练过程中，A的loss曲线比B的loss曲线更加的平缓，说明A模型更稳定。
% 3. 此外，A的loss比B的loss下降的更快，这说明了A模型在收敛速度优于B模型，1.4倍？。
% 具有更快的收敛速度。
% 4. 更进一步，我们记录了它们的运行时间，发现A比B快了1.4倍。
We compare our FL-tuning with PV2 in terms of convergence speed and training stability. The results on WSC, OCNLI, CLUENER, and CMRC2018 are shown in Figure \ref{fig:loss}. From the results, we observe that the loss value of FL-tuning decreases faster than that of PV2, which demonstrates that the former is better than the latter in terms of convergence speed. We further record their running time and find that our FL-tuning converges about 1.17 times faster than PV2 on average. In addition, the loss curve of FL-tuning is flatter than that of PV2 during the training process, indicating that FL-tuning is more stable.

% 1. 参数敏感性。
% 1. 大多数模型输出的结果往往对于学习率这一参数较为敏感。
% 2. 即使差距很小的两个学习率也会给预测结果带来较大的差异。
% 3. 因此，我们分析不同学习率对模型性能的影响。
% 4. 实验如图所示。
% 5. 从结果中，我们观察到在我们的模型中，不同学习率所得到的结果比较集中，而在B中，其结果比较分散。
% 6. 这验证了我们的模型对学习率不敏感。
\textbf{Hyperparameter insensitive.} The prediction results of deep models are often sensitive to hyperparameters, especially the learning rate. Even two learning rates with a small gap would make a great difference in predictions. Hence, we analyze the impact of different learning rates on model performance. The results on WSC and CLUENER datasets are reported in Figure \ref{fig:hyperparameter}. From the results, we find that in our FL-tuning, the results obtained with different learning rates are more concentrated, while the results are more scattered in PV2. This verifies that our model is more insensitive to the learning rate than PV2.

% neuron的个数
% 1. 增加的神经元的个数是FL-tuning中一个重要的超参。
% 2. 为了分析它对模型性能的影响，我们改变超参的取值。
% 3. 实验如图所示。
% 4. 从图中，我们观察到在A数据集上，模型的性能随着个数的增加而提升，而在另外三个数据集中，模型的性能与神经元的个数没有呈现出显性的规律。
% % 5. 总的来说，神经元的最优个数视任务而定的。
% 5. 模型性能并不会随着神经元的个数增加而提升。

\textbf{The number of hidden units.} The number of added hidden units is an influential hyperparameter for our FL-tuning. We adjust the values to analyze its impact on the model performance. The results on four datasets are shown in Figure \ref{fig:plength}. The model performance improves with the increase of the number of hidden units on CLUENER, while there is no apparent regularity on the other three datasets. As a result, the optimal value of the number of hidden units varies from task to task.

% neuron深度
% 1. 为了分析神经元加在不同层上的影响，我们选择k层in both ascending and descending order来添加神经元。
% 2. 实验结果如图所示。
% 3. 从图中，我们可以得到模型性能与所增加的层数呈现出了正相关的规律。
% 4. 此外，一般情况下，神经元加的层次越深，模型性能越好。
\textbf{Depth of hidden units.} To analyze the influence of adding hidden units to different layers on model performance, we select $k$ layers in both ascending and descending order to introduce hidden units. The results on TNEWS, WSC, OCNLI, and CLUENER are shown in Figure \ref{fig:PromptDepth}. According to the results, we observe that the performance of our FL-tuning is not only positively correlated to the number of added layers, but also related to the layer in which the added hidden units are located. The deeper the layer, the better the model performance.

% IF
% TNEWS 1.1
% WSC 1.1
% AFQMC
% CMNLI
% CSL
% CLUENER
% CHID
% C3 1.0
% CMRC2018

\subsection{Discussion}
% Discussion
% FFN 有5个比attention，剩下两个attention比FFN好
% 1. Transformer主要有两层组成：A和B。
% 2. 因此，FL-tuning的思想可以迁移到attention中。
% 3. 接下来，我们将在A上进行layer tuning，命名为C，来比较C与FL-tuning的性能。
% 4. 实际上，怎么做。
% ((Q,K) V) O
% (Q,K)
% (V,O)13*12=156个参数
% 5. 对比结果如图所示。
% 6. 从图中，我们看到在大多数情况在，A比B展示出了更好的性能。
% 7. 可能的原因是，正如Section 1所述，A占了Transformer中2/3的参数，对它进行微调的效果可能会更好。
% 8. 此外，实验结果展示出了同时在A和B中进行layer tuning的可行性。
% 9. 因此，在未来工作中，我们对该思想进行实现，并进一步地分析其可能存在的最优组合方式。
In the Transformer, there are mainly two sub-layers: multi-head self-attention and FFN. The idea of FL-tuning can also be transferred to self-attention. Thus, we perform layer tuning on multi-head self-attention called MA-tuning to compare its performance with FL-tuning.
That is, we increase the hidden size of the matrices in multi-head self-attention.
% We increase the hidden size of the value matrix $V$ in the self-attention from 64 to 77. Correspondingly, the hidden size of the $W^O$ matrix increases by the $h$ (the number of head in multi-head self-attention) times of the increase in $V$, which is an increase of 156. 
The comparison results are listed in Table \ref{tab:FLvsAP}. We find that FL-tuning slightly outperforms MA-tuning in most cases. The possible reason is that, as mentioned in the Introduction, FFN occupies about $2/3$ of the number of parameters in the model and tuning it may perform better. In addition, the results also show the feasibility of layer tuning on self-attention and FFN at the same time. Hence, we will implement this idea in the future and further analyze its possible optimal combination.

\section{Conclusion}
\label{section:conclusion-future-work}
In this paper, we propose a novel tuning way, namely layer tuning, to accommodate PLMs to downstream tasks, which aims to add learnable parameters to the Transformer layers. Specifically, we mainly focus on layer tuning on FFN called FL-tuning in the Transformer. It introduces additional units into the hidden layer of each FFN. We conduct extensive experiments on the CLUE benchmark and the results show that: 1) With only about 3\% of Transformer's parameters to be trained, our FL-tuning is comparable with fine-tuning on most datasets, and significantly outperforms fine-tuning on several datasets. 2) FL-tuning is better than prompt tuning methods under both full-data and few-shot settings in almost all cases. 3) FL-tuning is more stable and converges about 1.17 times faster than P-tuning v2.

{
\small
\bibliographystyle{unsrt}
\bibliography{neurips_2022}
}

\end{document}